\documentclass{article}

\PassOptionsToPackage{numbers, compress}{natbib}
\usepackage[preprint]{neurips_2026}

\usepackage[utf8]{inputenc} 
\usepackage[T1]{fontenc}    
\usepackage{hyperref}       
\usepackage{url}            
\usepackage{booktabs}       
\usepackage{amsfonts}       
\usepackage{nicefrac}       
\usepackage{microtype}      
\usepackage{xcolor}         
\usepackage{soul}
\usepackage{subcaption} 
\usepackage{multirow}
\usepackage{amsmath}
\usepackage{algorithm}
\usepackage{algpseudocode}
\usepackage{cleveref}
\usepackage{booktabs}  
\usepackage{amsthm}
\usepackage{graphicx}
\usepackage{wrapfig}
\usepackage{booktabs}
\usepackage{caption}
\newtheorem{definition}{Definition}

\newtheorem{proposition}{Proposition}
\newtheorem{theorem}{Theorem}

\title{Bernoulli Flow Models: Self-Consistent Generative Modeling for Binary Data}

\author{%
  \textbf{Hao Mo$^{1}$, Liying Yang$^{1}$, Shumin Yao$^{2}$, Xinxing Yu$^{1}$,} \\
  \textbf{Ajian Liu$^{3}$\thanks{Corresponding authors.}, Xudong Mao$^{4}$, Yanyan Liang$^{1}$\footnotemark[2]} \\
  \normalfont
  $^{1}$School of Computer Science and Engineering, Macau University of Science and Technology \\
  $^{2}$Department of Broadband Communication, Pengcheng Laboratory \\
  $^{3}$Institute of Automation, Chinese Academy of Sciences \\
  $^{4}$School of Artificial Intelligence, Sun Yat-sen University
}

\begin{document}

\setcounter{footnote}{1} 
\maketitle
\setcounter{footnote}{0} 

\begin{abstract}
Binary diffusion models typically require a massive number of function evaluations (NFEs) to generate high-quality samples, making practical inference computationally expensive. Reducing NFEs while preserving sample quality without relying on distillation or additional training remains a significant challenge. However, existing binary diffusion models sequentially define a discrete one-step forward path and subsequently derive the reverse posterior. In low-NFE scenarios that require cross-step sampling, these models incorrectly approximate the true multi-step likelihood using a single-step likelihood transition, which severely degrades sample quality. To address this fundamental limitation and completely decouple the generative dynamics from fixed discrete time steps, we propose Bernoulli Flow Models (BFM). Rather than building upon sequential one-step Markov diffusion chains, we predefine a unified continuous global Bernoulli probability flow path between data distributions and pure noise, from which we derive analytical closed-form posterior transitions over arbitrary time intervals. Consequently, reducing the inference NFE in BFM is no longer an approximation of skipped discrete steps; it simply requires re-evaluating the analytical posterior on the new time intervals. This eliminates the structural training-inference mismatch inherent to discrete chains, yielding strictly self-consistent low-NFE sampling. Experimental results show that BFM is highly robust to aggressive NFE reduction. On the LSUN Churches 256x256 dataset, a 256-step-trained BFM achieves an FID of 9.22 when sampled with only 16 steps, whereas the state-of-the-art discrete baseline severely degrades to 204.10. BFM also remains competitive with both continuous and discrete generative baselines under standard full-step inference. Ultimately, these results establish BFM as a theoretically rigorous, self-consistent, and practically effective framework for fast binary data generation.
Code is available at \url{https://github.com/minimo6666/BFM_NeurIPS2026}.
\end{abstract}

\section{Introduction}

Generative modeling over binary variables is a fundamental problem in binary image modeling, binarized latent representation learning, and statistical physics systems with two-state configurations. Recent binary generative models have shown that diffusion-style methods can achieve strong sample quality in binary spaces, including bit-based continuous diffusion formulations \cite{chen2022analog}, Binary Latent Diffusion (BLD) \cite{wang2023binary}, and recent binary diffusion models \cite{kinakh2025binary, wolleb2024binary}. However, these models still typically require hundreds or even thousands of sequential function evaluations (NFEs) to generate high-fidelity samples, making practical inference computationally expensive.

A natural objective is to reduce the inference NFE while preserving sample quality without additional distillation, consistency training, or task-specific retraining. In continuous diffusion models, this problem has been extensively studied through training-based acceleration, such as progressive distillation \cite{salimans2022progressive} and consistency models \cite{song2023consistency}, as well as training-free samplers, such as DDIM \cite{song2020denoising} and DPM-Solver \cite{lu2022dpm}, which evaluate the same continuous-time dynamics on a coarser inference grid \cite{ho2020denoising, song2020score, lipman2022flow}. In binary state spaces, however, such training-free acceleration is far less straightforward, because many binary diffusion models are defined through fixed one-step discrete Markov transitions rather than an arbitrary-interval continuous path.

The key difficulty is that low-NFE binary sampling is not merely a matter of skipping steps. Under the original full-step inference schedule, the reverse process remains aligned with the training-time Markov chain. Once the inference grid is coarsened, each reverse update must jump across multiple intermediate training intervals. For example, sampling a 256-step-trained binary model with only 16 steps requires each inference step to represent a large cross-step transition. A valid low-NFE sampler therefore needs a reverse posterior defined over arbitrary time intervals; reusing the original one-step posterior for such enlarged intervals no longer describes the same generative process.

This exposes a fundamental self-consistency issue in existing binary diffusion samplers. Although multi-step transitions may be derived in principle for a discrete Markov chain, practical binary diffusion samplers are often tied to the local transition rule used during training. When the sampling schedule changes after training, the required cross-step posterior is approximated by local one-step likelihood transitions, introducing a structural mismatch between the training-time process and the low-NFE inference-time process. As the step budget becomes smaller, this mismatch is amplified, leading to severe degradation in sample quality. Thus, the central obstacle to fast binary generation is not simply the lack of a faster numerical solver, but the lack of an analytical posterior mechanism that remains valid under changed inference grids.

Recent continuous-time discrete denoising and discrete flow-matching frameworks provide important foundations for discrete generative modeling \cite{campbell2022continuous, sun2022score, lou2023discrete, campbell2024generative, gat2024discrete, cheng2024categorical, davis2024fisher}. Nevertheless, these general formulations do not directly resolve the Bernoulli-specific low-NFE problem considered here: deriving closed-form posterior transitions over arbitrary time intervals so that changing the inference NFE becomes a rigorous reparameterization of the same binary generative process, rather than an approximate cross-step sampler.

To address this limitation, we propose \textbf{Bernoulli Flow Models (BFM)}, a self-consistent generative framework for binary data. Instead of defining generation through fixed one-step Markov transitions, BFM starts from a unified continuous Bernoulli probability path between data and pure noise. From this path, we derive closed-form marginal, transition, and posterior distributions over arbitrary time intervals. Consequently, reducing the number of inference steps only requires re-evaluating the same analytical posterior on the new time intervals, which decouples the generative dynamics from a fixed training-time discretization and makes low-NFE sampling self-consistent by construction.
Our contributions are summarized as follows:
Our contributions can be summarized as follows:
\begin{itemize}
    \item \textbf{A Closed-Form Bernoulli Generative Framework:}
    We propose Bernoulli Flow Models (BFM), which predefine a continuous Bernoulli probability path and derive analytical closed-form marginal, transition, and posterior distributions from this path.

    \item \textbf{Self-Consistent Low-NFE Inference:}
    The resulting arbitrary-interval posterior enables principled cross-step updates when the inference NFE is reduced after training. Instead of approximating a large multi-step jump with local one-step rules, BFM recomputes the analytical posterior on the new time intervals, yielding self-consistent low-NFE sampling.

    \item \textbf{Effective Training-Free Acceleration in Practice:}
    Experiments show that BFM is highly robust to aggressive NFE reduction. From a single 256-step checkpoint on LSUN Churches ($256 \times 256$), BFM maintains an FID of 9.22 with only 16 sampling steps, whereas the corresponding baseline degrades to 204.10. Additional results show that BFM remains competitive under the original full-step inference setting across diverse tasks.
\end{itemize}

\section{Related Work}

\textbf{Binary and Bit-Level Flow/Diffusion Models.}
Binary generative modeling has recently received increasing attention. Early work already explored diffusion-style modeling on binary state spaces \cite{sohl2015deep}. Analog Bits \cite{chen2022analog} proposed a bit-encoded continuous diffusion approach, where discrete variables are represented as binary bits and then modeled as real-valued analog bits by a continuous diffusion model. Binary Latent Diffusion (BLD) \cite{wang2023binary} more directly studied diffusion-based generation in binarized latent spaces and demonstrated strong high-resolution image generation performance. Other recent works further investigated Bernoulli or binary diffusion formulations, such as masked Bernoulli diffusion for anomaly detection \cite{wolleb2024binary} and Binary Diffusion Probabilistic Models (BDPM) \cite{kinakh2025binary}. Concurrently, Binary Flow Matching \cite{hong2026binary} studies robust flow matching on binary manifolds from the perspective of prediction--loss space alignment, showing that aligning the training objective with the signal space can remove singular weighting and improve training stability. These works demonstrate the promise of binary or bit-level diffusion/flow modeling. However, their focus is mainly on representation, training objectives, or binary-space generative quality, while the problem of self-consistent low-NFE inference from a fixed checkpoint remains underexplored.

\textbf{Fast Sampling and Training-Free Acceleration.}
A major limitation of diffusion-based generative models is their slow iterative sampling, which typically requires many sequential function evaluations (NFEs). In continuous domains, extensive work has studied acceleration. Training-based methods, such as progressive distillation \cite{salimans2022progressive} and consistency models \cite{song2023consistency}, reduce sampling steps through an additional training stage. Training-free samplers, such as DDIM \cite{song2020denoising} and DPM-Solver \cite{lu2022dpm}, instead accelerate sampling directly from a pretrained model by evaluating the reverse dynamics on a coarser inference grid. These techniques have made low-NFE generation practical for continuous diffusion models. However, binary diffusion models cannot directly rely on continuous-domain ODE or SDE solvers, especially when the reverse process is defined through discrete Bernoulli transitions. Reducing the NFE then requires valid cross-step posterior updates over enlarged time intervals, rather than simply applying a continuous-domain solver. This makes training-free acceleration in binary generative modeling substantially less developed.

\textbf{Relation to General Discrete Diffusion and Flow Models.}
A separate line of work studies generative modeling over broader discrete or categorical state spaces, including discrete diffusion models \cite{austin2021structured}, continuous-time discrete denoising frameworks \cite{campbell2022continuous, sun2022score, lou2023discrete}, and recent discrete flow-matching formulations \cite{campbell2024generative, gat2024discreteflow, cheng2024categorical, davis2024fisher}. These works provide important general foundations for discrete generative modeling. Our work is complementary but more specialized: rather than proposing a generic categorical diffusion or flow construction, we focus on Bernoulli binary variables and derive closed-form posterior transitions over arbitrary time intervals, enabling self-consistent low-NFE inference from a single pretrained checkpoint.

\section{Bernoulli Flow Models}
\label{sec:method}

In this section, we present the theoretical framework of Bernoulli Flow Models (BFM). We detail how a unified continuous global Bernoulli probability flow path is established to completely decouple the generative dynamics from fixed discrete time steps, and formally derive the resulting analytical probability distributions.

\subsection{Notation and Problem Formulation}
\label{subsec:Notation}

Consider an arbitrary binary data distribution $p_\mathbf{X}(\mathbf{x})$ over the sample space $\mathcal{X} = \{0, 1\}^d$, where $d \in \mathbb{N}^+$ denotes the dimensionality, and $\mathbf{X}$ is a discrete random vector. We denote the Bernoulli distribution as $\mathcal{B}$.

Given target dataset observations $X_0 \sim \pi_0$ and a prior $X_1 \sim \pi_1$ (where $\pi_1 = \mathcal{B}(0.5)$ represents binary white noise), we define an optimal transport-style Bernoulli probability path. Analogous to linear interpolation in continuous spaces, we construct the path directly in the Bernoulli parameter space across continuous time $t_\tau \in [1, 0]$:
\begin{equation}
    X_{t_{\tau}} \sim  \mathcal{B}(t_{\tau}X_0 + (1-t_{\tau})X_1), \quad \tau = 0, 1, \dots, T,
    \label{eq:ot_interpolation}
\end{equation}	
where $X_{t_{0}} \equiv X_0$ and $X_{t_{T}} \equiv X_1$. The random variable $X_{t_{\tau}}$ strictly follows the Bernoulli probability flow path from the data to the prior.

\subsection{Bernoulli Probability Flow Dynamics}
\label{subsec:dynamics}

Building upon the optimal transport-style Bernoulli probability path defined above, and assuming the prior distribution $\pi_1$ is a uniform Bernoulli distribution $\mathcal{B}(0.5)$ (representing binary white noise), the marginal transition probability from the target observation $X_0$ (or $X_{t_{0}}$) to an arbitrary intermediate state $X_{t_{\tau}}$ can be analytically expressed as:
\begin{equation}
  \begin{split}
    q(X_{t_{\tau}}|X_{t_{0}}) &=  \mathcal{B}(X_{t_{\tau}};\alpha_{t_{\tau}}), \\ 
    \alpha_{t_{\tau}} &= g\left(X_{t_{0}},\frac{1-t_{\tau}}{2}\right), \\
    g(a,b) &= a(1-b) + b(1-a).
  \end{split}
    \label{eq:marginal_x0_xt}
\end{equation}
For $0\leq t_\tau<t_{\tau-1}\leq1$, by exploiting the Markov property along the constructed path, we can further derive the one-step forward transition probability between sequential states $X_{t_{\tau-1}}$ and $X_{t_{\tau}}$:
\begin{equation}
  \begin{split}
    q(X_{t_{\tau}}|X_{t_{\tau-1}}) & = \mathcal{B}(X_{t_{\tau}};g(X_{t_{\tau-1}},\gamma_{t_{\tau}})), \\
  \gamma_{t_{\tau}} &= \frac{0.5(t_{\tau-1} - t_{\tau})}{t_{\tau-1}},
  \end{split}
    \label{eq:one_step_transition}
\end{equation}
where $\gamma_{t_{\tau}}$ denotes the explicit flip probability from $X_{t_{\tau-1}}$ to $X_{t_{\tau}}$. 

Using Bayes' theorem, the ground-truth posterior distribution of the previous state $X_{t_{\tau-1}}$ conditioned on the current state $X_{t_{\tau}}$ and the source data $X_{t_{0}}$ is formulated, for $x\in\{0,1\}$, as:
\begin{equation}
q(X_{t_{\tau-1}}=x\mid X_{t_{\tau}},X_{t_0})
=
\frac{
q(X_{t_{\tau}}\mid X_{t_{\tau-1}}=x)
q(X_{t_{\tau-1}}=x\mid X_{t_0})
}{
\sum_{x'\in\{0,1\}}
q(X_{t_{\tau}}\mid X_{t_{\tau-1}}=x')
q(X_{t_{\tau-1}}=x'\mid X_{t_0})
},
\label{eq:bayes}
\end{equation}
which simplifies into the analytical closed-form solution:
\begin{equation}
  q(X_{t_{\tau-1}}|X_{t_{\tau}}, X_{t_{0}}) =  \mathcal{B}\left(X_{t_{\tau-1}}; \frac{g(X_{t_{\tau}},\gamma_{t_{\tau}})\alpha_{t_{\tau-1}}}{g(g(X_{t_{\tau}},\gamma_{t_{\tau}}),1-\alpha_{t_{\tau-1}})}\right).
    \label{eq:posterior}
\end{equation}

Similar to the $x_0$-prediction objective in continuous DDPM \cite{ho2020denoising}, we parameterize our generative flow model $f_\theta$ to directly predict the uncorrupted data $X_{t_{0}}$. This design aligns the model output with the analytical Bernoulli posterior probability path, thereby enabling the generative process to sample back to the target distribution $\pi_0$ from the prior $\pi_1$. The backward Bernoulli flow transition step is parameterized as:
\begin{equation}
  \begin{split}
  p_{\theta}(X_{t_{\tau-1}}|X_{t_{\tau}}) 
  = \mathcal{B}\left(X_{t_{\tau-1}}; \frac{g(X_{t_{\tau}},\gamma_{t_{\tau}})g\left(f_\theta(X_{t_{\tau}},t_{\tau}), \frac{1-t_{\tau-1}}{2}\right)}{g\left(g(X_{t_{\tau}},\gamma_{t_{\tau}}),1-g\left(f_\theta(X_{t_{\tau}},t_{\tau}),\frac{1-t_{\tau-1}}{2}\right)\right)}\right).
  \end{split}
\label{eq:backward_transition_params}
\end{equation}

Here, $f_\theta$ is parameterized to directly predict the uncorrupted data $X_{t_{0}}$. Following the standard $x_0$-prediction paradigm widely adopted in diffusion and flow matching models, we employ binary cross-entropy (BCE) to formulate the training objective:

\begin{equation}
  \begin{split}
  \mathcal{L}_{\text{BCE}}(\theta) &:= -\mathbb{E}_{t_{\tau},X_{t_{0}},X_{t_{\tau}}} \big[ X_{t_{0}} \log \left( f_\theta(X_{t_{\tau}},t_{\tau}) \right) \\
  & \quad + (1 - X_{t_{0}}) \log \left(1 - f_\theta(X_{t_{\tau}},t_{\tau})  \right) \big].
  \end{split}
  \label{eq:bce_loss}
\end{equation}

Detailed proofs for the derivations in this section are provided in Appendix \ref{appendix_proofs}.

\subsection{Self-Consistent Cross-Step Sampling}
\label{subsec:cross_step}

A fundamental advantage of BFM emerges during fast inference when the number of function evaluations (NFEs) is aggressively reduced. Suppose we change the inference time grid and need to jump from $t_{\tau}$ directly to a distant step $t_{\tau-k}$ ($k > 1$). In existing heuristic chain-based models, navigating this large stride requires accumulating multiple local one-step transitions, leading to severe approximation errors.

In contrast, BFM naturally supports analytical cross-step posterior sampling. Because the marginal parameters $\alpha$ and the transition flip probabilities $\gamma$ in our framework are continuous analytical functions of time $t \in [0,1]$, evaluated in the decreasing direction from data to noise, we can directly compute the analytical transition over any arbitrary interval. By generalizing Eq.~\ref{eq:one_step_transition}, the analytical cross-step forward transition from $t_{\tau-k}$ to $t_{\tau}$ is uniquely determined by:
\begin{equation}
    \gamma_{t_{\tau} | t_{\tau-k}} = \frac{0.5(t_{\tau-k} - t_{\tau})}{t_{\tau-k}}.
    \label{eq:cross_step_gamma}
\end{equation}

By substituting $\gamma_{t_{\tau} | t_{\tau-k}}$ and the corresponding marginal $\alpha_{t_{\tau-k}}$ into Bayes' theorem, we establish the following core theoretical guarantee of our framework:

\begin{theorem}[Analytical Cross-Step Posterior]
\label{theorem:cross_step}
For any reverse interval satisfying $0\leq t_{\tau}<t_{\tau-k}\leq1$ (and hence $t_{\tau-k}>0$), where $k>1$, the analytical cross-step posterior of the Bernoulli Flow Model admits a perfect closed-form solution:
\begin{equation}
    q(X_{t_{\tau-k}} \mid X_{t_{\tau}}, X_{t_{0}}) =  \mathcal{B}\left(X_{t_{\tau-k}}; \frac{g(X_{t_{\tau}},\gamma_{t_{\tau} | t_{\tau-k}})\alpha_{t_{\tau-k}}}{g(g(X_{t_{\tau}},\gamma_{t_{\tau} | t_{\tau-k}}),1-\alpha_{t_{\tau-k}})}\right).
    \label{eq:cross_step_posterior}
\end{equation}
\end{theorem}

\begin{proof}
The rigorous derivation of Theorem \ref{theorem:cross_step} is provided in Appendix \ref{appendix_proofs}.
\end{proof}

During fast sampling, we re-parameterize the
decreasing schedule of $t\in[0,1]$ from $1$ to $0$ to the new NFE
budget $K$ (where $K\ll T$).
We refer to this property as the
\textbf{self-consistency of BFM under cross-step sampling}.
This establishes the structural consistency of
the oracle posterior under arbitrary inference grids.
The approximation errors introduced by the learned
predictor are analyzed in the following section.

\subsection{Error Bounds for Learned Reverse Sampling}
\label{sec:learned_error}

Theorem~\ref{theorem:cross_step} establishes the exact
cross-step posterior conditioned on the clean state
$X_{t_0}$. In practice, $X_{t_0}$ is replaced by
$f_\theta(X_t,t)$. We characterize the resulting error
for both one-step and cross-step sampling.

\begin{proposition}[Learned Reverse Transition Error]
\label{prop:learned_error}
For any reverse interval $0\leq t<s<1$, where
$s=t_{\tau-k}$ and $t=t_\tau$, the coordinate-wise
total variation (TV) distance between the oracle
and learned posteriors satisfies
\begin{equation}
\begin{aligned}
&D_{\mathrm{TV}}\big(
q(X_s\mid X_t,X_{t_0}),
p_\theta(X_s\mid X_t)
\big)\\
&\quad\leq
\frac{s^2-t^2}{s(1-t)^2}
\left|f_\theta(X_t,t)-X_{t_0}\right|.
\end{aligned}
\label{eq:learned_tv_bound}
\end{equation}
\end{proposition}

The corresponding KL divergence is also bounded by an
interval-dependent constant times the binary cross-entropy
(BCE) loss, connecting the training objective to the
learned reverse-transition error.

\begin{proposition}[Accumulated Sampling Error]
\label{prop:accumulated_error}
For an inference grid $1=t_0>\cdots>t_K=0$,
the final sampling error satisfies
\begin{equation}
\begin{aligned}
&D_{\mathrm{TV}}\big(
\pi_0,p_\theta(X_{t_0})
\big)\\
&\quad\leq
\min\left\{
1,\,
E_{\mathrm{rec}}+
\sum_{j=2}^{K}
\frac{t_{j-1}^{2}-t_j^{2}}
{t_{j-1}(1-t_j)^2}
\varepsilon_j
\right\},
\end{aligned}
\label{eq:accumulated_tv_bound}
\end{equation}
\end{proposition}

Here, $\varepsilon_j$ bounds the expected soft
$\ell_1$ prediction error under the forward distribution,
and $E_{\mathrm{rec}}$ denotes the final reverse-step error.
These bounds distinguish the exact oracle cross-step
posterior from the learned sampler, whose approximation
error depends on prediction accuracy and inference
intervals. Detailed proofs are provided in
Appendix~\ref{app:learned_error_bounds}.

\subsection {Model Training and Sampling}

\begin{algorithm}
\caption{Training procedure.}
\label{alg:training}
\begin{algorithmic}[1]
\State \textbf{Given}: Binary diffusion model $f_\theta$ parametrized by $\mathcal{T}_\theta$; A dataset $\mathbf{X}$.
\State \textbf{Given}: Diffusion steps $T$; Probability path defined by \cref{eq:ot_interpolation}; Training steps $I$.
\State Initializing $\mathcal{T}_\theta$.
\For{Step $i = 1 : I$}
    \State Sampling data $\mathbf{x}_{t_{0}}\sim \mathbf{X}$, and time step $\tau \sim \{1, \ldots, T\}$.
    \State Obtaining $\mathbf{x}_{t_{\tau}}$ using $\mathbf{x}_{t_{0}}$, $t_{\tau}$ , and \textcolor{blue}{probability path with \cref{eq:ot_interpolation}}.
    \State Predicting the probability that the state is $\mathbf{x}_{t_{0}}$ using $f_\theta(\mathbf{x}_{t_{\tau}}, t_{\tau}) = \sigma(\mathcal{T}_\theta(\mathbf{x}_{t_{\tau}}, t_{\tau})/\kappa)$.
    \State Calculating loss $\mathcal{L}$ using \cref{eq:bce_loss}.
    \State Backpropagating $\mathcal{L}$ and updating $\theta$.
\EndFor
\State \textbf{Return} Binary diffusion model $f_\theta$.
\end{algorithmic}
\end{algorithm}

\begin{algorithm}
\caption{Sampling procedure.}
\label{alg:sampling}
\begin{algorithmic}[1]
\State \textbf{Given}: Trained binary diffusion model $f_\theta$;The sample shape dimension $shape$.
\State \textbf{Given}: Diffusion steps $T$; \textcolor{blue}{Probability path defined by \cref{eq:ot_interpolation}}; Temperature $\kappa$.
\State Sampling $\mathbf{x}_{t_{T}} = \text{Bernoulli}(\mathbf{x}^{\text{init}})$, where $\mathbf{x}^{\text{init}} \in \mathbb{R}^{shape}$ and contains 0.5 only.
\For{Step $\tau = T : 1$}
    \State \textcolor{blue}{Predicting $p_\theta(\mathbf{x}_{t_{\tau-1}})$ with $f_\theta(\mathbf{x}_{t_{\tau}}, t_{\tau}) = \sigma(\mathcal{T}_\theta(\mathbf{x}_{t_{\tau}}, t_{\tau})/\kappa)$ and \cref{eq:backward_transition_params}}.
    \State Sampling $\mathbf{x}_{t_{\tau-1}} = \text{Bernoulli}(p_\theta(\mathbf{x}_{t_{\tau-1}}))$
\EndFor
\State $\hat{\mathbf{x}}_{t_{0}} = \mathbf{x}_{t_{\tau-1}}$
\State \textbf{Return} the final sample $\hat{\mathbf{x}}_{t_{0}}$.
\end{algorithmic}
\end{algorithm}

We now describe how BFM can be trained and sampled by minimal modification of existing binary diffusion \cite{wang2023binary} training and sampling architectures, as summarized in Algorithms \ref{alg:training} and \ref{alg:sampling}, where we highlight the differences between BFM and existing binary diffusion models in blue. The prediction of $X_{t_{0}}$ is implemented via a neural network $f_\theta(X_{t_{\tau}}, t_{\tau}) = \sigma(\mathcal{T}_\theta(X_{t_{\tau}}, t_{\tau})/\kappa)$, where $\sigma(\cdot)$ is the sigmoid function, $\mathcal{T}_\theta$ is a neural network with parameters $\theta$, and $\kappa$ is a temperature hyperparameter.
The temperature is used to control the diversity of the generated samples. A smaller temperature leads to less diverse samples, while a larger temperature increases diversity.

\section{Experiments}
\label{experiments}

We evaluate BFM from four perspectives. First, we validate our central claim that BFM maintains stability under \emph{training-free} NFE reduction using a single pretrained checkpoint. Second, we verify its competitiveness on high-dimensional latent image benchmarks under standard sampling. Third, we demonstrate direct generative performance on binary manifolds via Binarized MNIST. Fourth, we conduct an ablation study to analyze the impact of different probability path schedulers. Furthermore, Appendix~\ref{appendix_ising} details BFM's applicability beyond image synthesis by modeling the Ising binary physical system.

\subsection{Experimental Setup}
\label{sec:setup}

To ensure a rigorous and fair comparison with BLD~\cite{wang2023binary}, we implement BFM directly on top of the official BLD codebase. We fully follow the experimental settings of BLD, using the identical full transformer backbone and hyperparameters. All experiments are conducted on two NVIDIA GeForce RTX~4090 GPUs. We use the Adam optimizer with $\beta_1 = 0.9$, $\beta_2 = 0.99$, and $\epsilon = 1 \times 10^{-8}$, together with a warm-up scheduler whose peak learning rate is $2 \times 10^{-4}$ after 10K iterations. Unless otherwise specified, we set the temperature $\kappa$ to $0.9$ as in the default BLD setup, and compute all FID, Precision, and Recall metrics from 50K generated samples compared against the corresponding training datasets.

\subsection{Self-Consistent Low-NFE Sampling}
\label{sec:low_nfe}

We first evaluate the main practical regime targeted by our method: \emph{training-free acceleration}, where the number of function evaluations (NFEs) is reduced at inference time without retraining or distillation. This setting directly tests whether BFM remains stable when the inference grid is coarsened after training.

To this end, we compare BFM and BLD on LSUN Churches $256\times256$ under a fixed-checkpoint protocol. Specifically, both models are trained with 256 sampling steps using the linear probability path scheduler, and we evaluate the 100K-iteration checkpoints while varying only the inference NFE. All other training and evaluation settings follow Sec.~\ref{sec:setup}.

\begin{figure*}[h]
    \centering
    \begin{subfigure}[t]{0.48\textwidth}
        \centering
        \includegraphics[width=\linewidth]{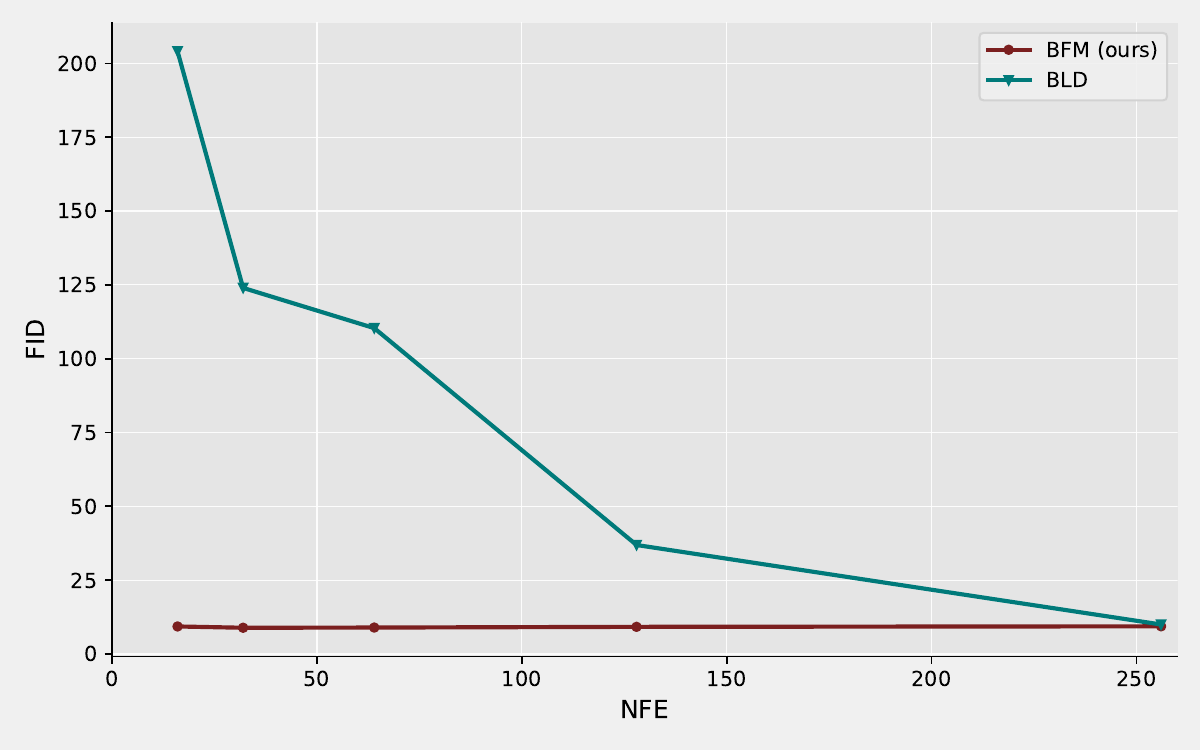}
        \caption{Low-NFE sampling robustness.}
        \label{fig:fid_low_nfe}
    \end{subfigure}
    \hfill
    \begin{subfigure}[t]{0.48\textwidth}
        \centering
        \includegraphics[width=\linewidth]{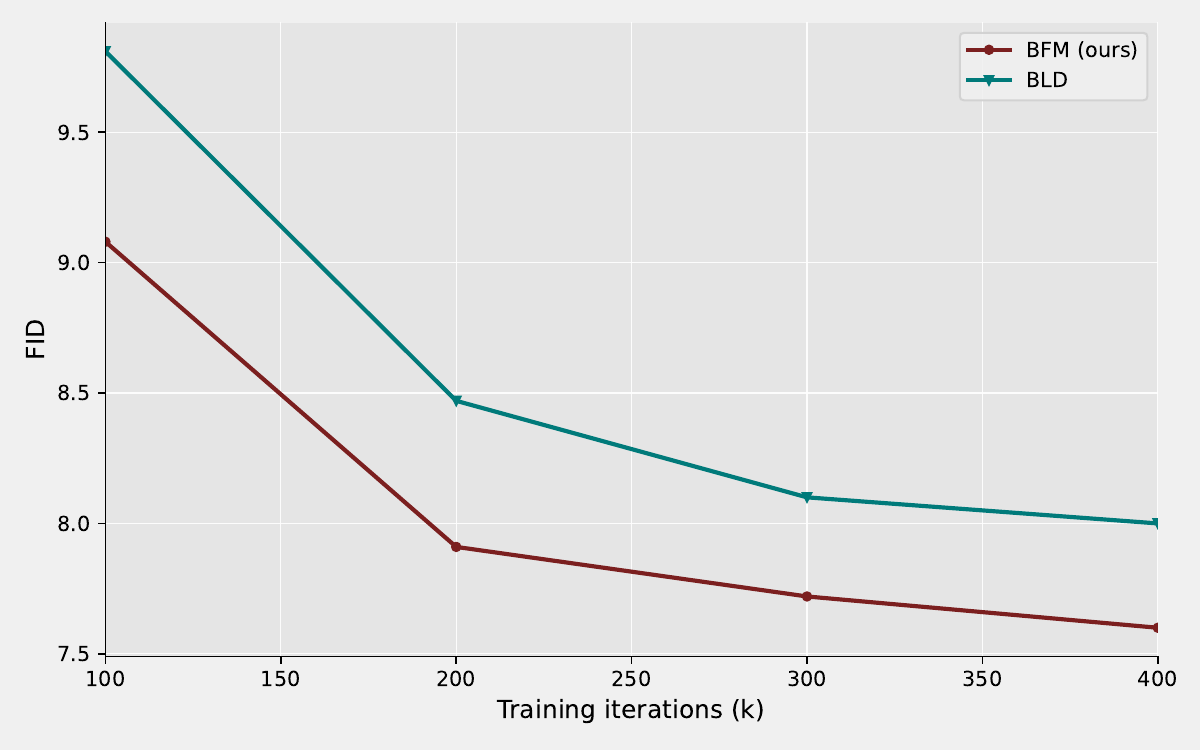}
        \caption{Training dynamics.}
        \label{fig:fid_vs_train_iterations}
    \end{subfigure}
    \caption{
    \textbf{Comparison of BFM and BLD under reduced inference steps and during training.}
    (a) FID versus NFE on LSUN Churches ($256\times256$), evaluated from the 100K-iteration checkpoints trained with 256 sampling steps using the linear probability path scheduler, where only the inference NFE is varied.
    (b) Training convergence comparison on LSUN Bedrooms ($256\times256$). We visualize the FID scores from 100K to 400K training iterations ($\times 10^3$). Both BFM and BLD are trained under identical experimental settings using a linear scheduler. BFM reaches lower FID consistently throughout this range.
    }
    \label{fig:fid_comparison}
\end{figure*}

Fig.~\ref{fig:fid_low_nfe} reports the core low-NFE result. Starting from the same 256-step training setup, BFM remains highly stable as the inference NFE is reduced, whereas BLD deteriorates rapidly under aggressive NFE reduction. In particular, when sampled with only 16 inference steps, BFM achieves an FID of 9.22 under the linear path scheduler, while the discrete baseline BLD severely degrades to 204.10. This result directly supports our central claim that BFM is substantially more robust than chain-based Bernoulli diffusion under training-free NFE reduction.

Fig.~\ref{fig:fid_vs_train_iterations} further compares the training dynamics from 100K to 400K iterations on LSUN Bedrooms $256\times256$. Under identical architectures, schedulers, and optimization settings, BFM consistently reaches lower FID values than BLD throughout this training range. This behavior is consistent with the structural design of BFM. In our framework, the supervision at different sampled time points is induced by the same closed-form Bernoulli probability path, so different sampled times correspond to different intervals of one unified transition family. In contrast, BLD is defined on a fixed one-step discrete chain, where reverse transitions are tied to specific grid steps. This path-consistent parameterization gives BFM a more coherent training signal across time, which empirically translates into faster FID reduction during training.

To make this behavior visually explicit, Fig.~\ref{fig:low_nfe_visual_comparison} compares samples generated from the respective 256-step checkpoints of BLD and BFM under different inference NFEs. As the NFE decreases, BLD quickly becomes blurry and structurally unstable, while BFM preserves global scene structure and semantic layout across a broad range of NFEs. This qualitative evidence is fully consistent with the quantitative FID curve.

\begin{figure*}[htbp]
    \centering
    \begin{subfigure}{0.48\textwidth}
        \centering
        \includegraphics[width=\linewidth]{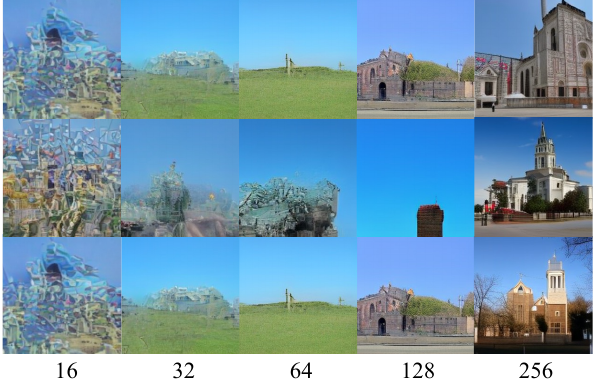}
        \caption{BLD}
        \label{fig:BLD_samples_main}
    \end{subfigure}
    \hfill
    \begin{subfigure}{0.48\textwidth}
        \centering
        \includegraphics[width=\linewidth]{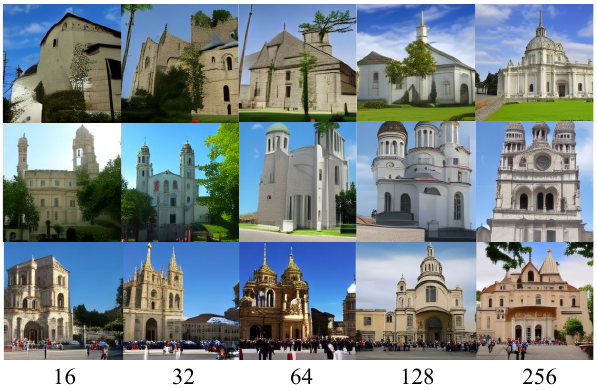}
        \caption{BFM}
        \label{fig:BFM_samples_main}
    \end{subfigure}
    \caption{
    Visual comparison of samples generated with different NFEs from the respective checkpoints of BLD and BFM, both trained with 256 sampling steps on LSUN Churches $256\times256$. For each method, each row corresponds to a different NFE setting, while images in the same column share the same initial latent code.
    }
    \label{fig:low_nfe_visual_comparison}
\end{figure*}

The practical significance of these results is that BFM supports \emph{training-free acceleration} from a single pretrained checkpoint. Reducing the NFE does not require retraining, distillation, or switching to a different sampler. Instead, the model remains stable because the reverse updates can be recomputed on the new coarser inference grid. A more detailed discussion of the corresponding cross-step posterior discrepancy in BLD and the self-consistency mechanism of BFM is deferred to Appendix~\ref{sec:cross_step_consistency}.

\subsection{Standard Sampling Performance on High-Dimensional Image Generation}

We next evaluate BFM under the standard sampling setting to examine its generation quality with the original inference budget. This experiment complements the low-NFE analysis and verifies that BFM remains competitive in the conventional full-step sampling regime.

\begin{table*}[h]
    \centering
    \resizebox{\textwidth}{!}{
    \begin{tabular}{ccccccccccc}
    \hline
    \multicolumn{1}{c}{\multirow{2}{*}{Methods}} & \multicolumn{1}{c}{Steps} & \multicolumn{3}{c}{LSUN-Bedrooms 256x256} & \multicolumn{3}{c}{LSUN-Churches 256x256}  & \multicolumn{3}{c}{FFHQ 256x256}\\ 
    \multicolumn{1}{c}{}          & {}                  & FID \(\downarrow\) & Prec. \(\uparrow\) & Recall \(\uparrow\) & FID \(\downarrow\) & Prec. \(\uparrow\) & Recall \(\uparrow\) & FID \(\downarrow\) & Prec. \(\uparrow\) & Recall \(\uparrow\)   \\ \cline{1-11}
    \textbf{Continuous Generative Models} & {} & {} & {} & {} & {} & {} & {} & {} & {} & {} \\ 
    StyleGAN \cite{karras2019style,karras2020analyzing} 
        & -   
        & \textbf{2.35}    
        & 0.59               
        & \textbf{0.48}  
        & \textbf{3.86}         
        & 0.60               
        & 0.43           
        & \textbf{4.16}  
        & 0.71            
        & \underline{0.46} \\
    LDM-4/8/4 \cite{rombach2022high}         
        & 200            
        & \underline{2.95}             
        & \underline{0.66}               
        & \textbf{0.48}      
        & \underline{4.02}              
        & \underline{0.64}                
        & \underline{0.52}          
        & \underline{4.98}   
        & \textbf{0.73}   
        & \textbf{0.50} \\
    Patch-DM \cite{ding2023patched}          
        & 50             
        & 6.04             
        & 0.56               
        & \underline{0.44}               
        & 5.49              
        & 0.62               
        & \textbf{0.53}  
        & 10.02          
        & 0.68            
        & 0.44 \\
    VQ-LCMD \cite{nguyen2024improving}       
        & 200            
        & 4.16             
        & \textbf{0.72}      
        & 0.40               
        & 4.99              
        & \textbf{0.75}      
        & 0.42           
        & 7.25           
        & \underline{0.72}            
        & \underline{0.46} \\
    \hline
    \textbf{Discrete Generative Models} & {} & {} & {} & {} & {} & {} & {} & {} & {} & {} \\ 
    D3PM \cite{austin2021structured}         
        & 200            
        & 6.60             
        & \underline{0.60}               
        & 0.35               
        & 6.02              
        & \underline{0.68}               
        & \underline{0.39}           
        & \underline{9.49}           
        & \textbf{0.71}            
        & 0.41 \\
    VQ-Diffusion \cite{gu2022vector}         
        & 200            
        & 7.19             
        & 0.54               
        & \textbf{0.37}               
        & 6.88              
        & \textbf{0.72}      
        & 0.37           
        & \textbf{8.79}           
        & \underline{0.70}            
        & 0.43 \\
    BLD*  \cite{wang2023binary}                                    
        & 64             
        & \underline{6.22}             
        & \textbf{0.69}      
        & \underline{0.36}               
        & \underline{5.48}              
        & 0.66      
        & \textbf{0.41} 
        & 11.46          
        & 0.68   
        & \underline{0.48} \\
    \textbf{BFM} (ours)                      
        & 64             
        & \textbf{5.86}    
        & \textbf{0.69}      
        & \textbf{0.37}      
        & \textbf{5.32}     
        & 0.66      
        & \textbf{0.41} 
        & 10.87 
        & 0.68   
        & \textbf{0.49}\\ 
    \hline 
    \end{tabular}
    }
    \caption{Comparison of various methods for image generation on LSUN Bedrooms, LSUN Churches, and FFHQ. All images are of resolution $256\times256$. BLD* denotes our reproduction of BLD using its official codebase under the same machine environment and training configuration as BFM. The reproduced BLD* and BFM results are evaluated from checkpoints trained for 800K iterations. Within each method category, the best and second-best results in each column are highlighted in \textbf{bold} and \underline{underline}, respectively.}
    \label{tab:lsun_bedroom_church_ffhq}
\end{table*}

To evaluate BFM on challenging image synthesis tasks, we use it as a latent-space generator. Following BLD~\cite{wang2023binary}, we employ the pretrained encoder and decoder to map images from LSUN Bedrooms and Churches~\cite{yu2015lsun}, and FFHQ~\cite{karras2019style} into a binary latent space of size $16 \times 16 \times 64$. Detailed implementation specifics of the autoencoder can be found in \cite{wang2023binary}.

A direct comparison to the originally reported BLD numbers is not fully controlled, because the corresponding training checkpoints are not publicly available and the original training environment differs from ours. Therefore, instead of mixing unmatched results in the main comparison, we reproduce BLD using its official codebase under the same machine environment and the same training configuration as BFM, including the backbone, optimizer, learning rate, loss function, and iteration budget. Table~\ref{tab:lsun_bedroom_church_ffhq} reports only this unified comparison. Under this strictly matched setting, BFM consistently outperforms the reproduced BLD across all three datasets (e.g., FID 5.86 vs.\ 6.22 on LSUN Bedrooms and 10.87 vs.\ 11.46 on FFHQ). 

Beyond the quantitative improvements presented in Table~\ref{tab:lsun_bedroom_church_ffhq}, we also provide a qualitative assessment to showcase BFM's ability to generate high-fidelity images across diverse datasets. Representative samples are shown in Fig.~\ref{fig:3_set_samples}. Additionally, Appendix~\ref{appendix_temperature_visualization} provides extended visualizations demonstrating how increasing the sampling temperature shifts the outputs from stable, high-fidelity structures to greater visual diversity.

\begin{figure*}[htbp]
    \centering
    \includegraphics[width=\textwidth]{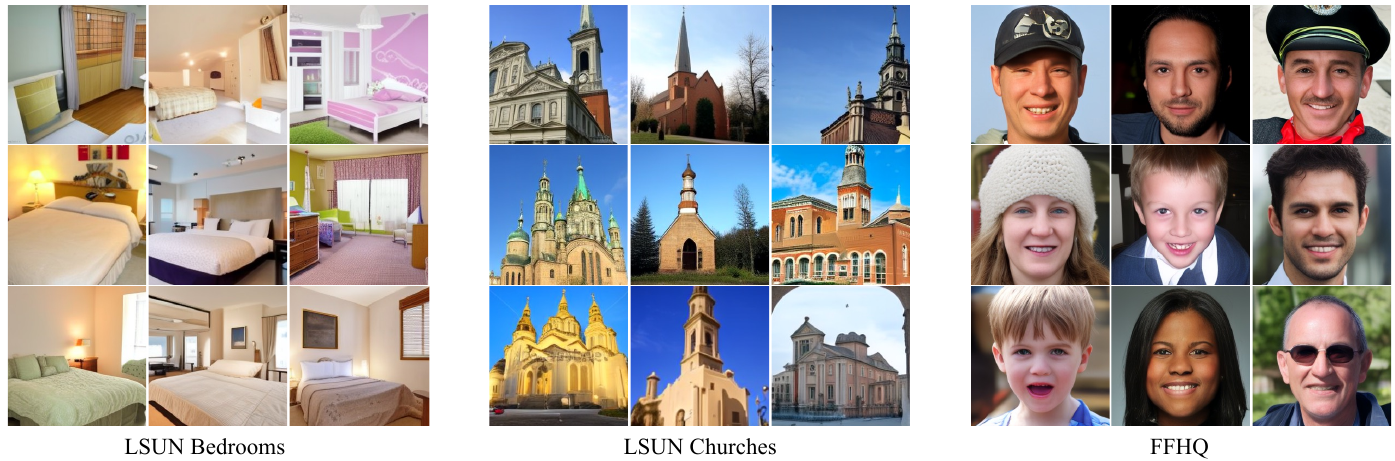}
    \caption{Samples from BFM on LSUN-Bedrooms, LSUN-Churches and FFHQ datasets. All samples resolution are 256x256.}
    \label{fig:3_set_samples}
\end{figure*}

\subsection{Performance on Binarized MNIST}

\begin{figure}[htbp]
    \centering
    \begin{minipage}[c]{0.15\textwidth}
        \centering
        \includegraphics[width=\linewidth]{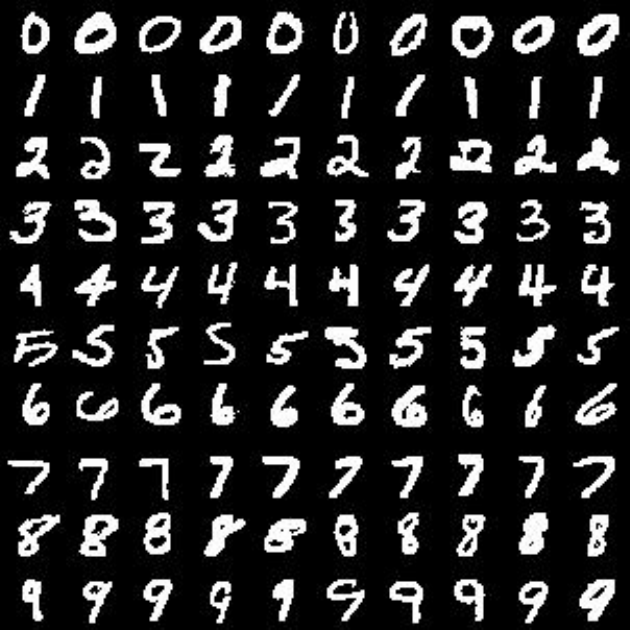}
    \end{minipage}
    \hspace{0.05\textwidth} 
    \begin{minipage}[c]{0.48\textwidth}
        \centering
        \resizebox{\linewidth}{!}{
        \begin{tabular}{|l|l|l|}
        \hline
        Model & Base method & FID \\ \hline
        SFM \cite{cheng2024categorical} & Flow Matching & 4.62 \\ \hline
        $\alpha$-Flow \cite{cheng2025alpha} & Flow Matching & 5.02 \\ \hline
        DMPM \cite{pham2025discrete} & Score-based Diffusion & 2.89 \\ \hline
        CoVAE \cite{silvestri2025covae} & VAE & \underline{0.58} \\ \hline
        \textbf{BFM} (ours) & Flow Matching \& Diffusion & \textbf{0.42} \\ \hline
        \end{tabular}
        }
    \end{minipage}
    
    \vspace{2mm} 
    \caption{Left: Binarized MNIST samples generated by BFM. Right: Comparison of FID scores. The best and second-best results are highlighted in \textbf{bold} and \underline{underline}, respectively.}
    \label{fig:mnist_performance}
\end{figure}

To validate the generative performance of BFM on high-dimensional binary manifolds, we conduct experiments on the binarized MNIST~\cite{deng2012mnist} dataset. Following standard practices, we apply a thresholding scheme where pixels greater than 0 are set to 1, and others are set to 0. We evaluate the performance of BFM and report the FID scores in Figure~\ref{fig:mnist_performance}, comparing it with other state-of-the-art methods. 

As shown, BFM achieves the best FID score of 0.42 among the evaluated models. The FID metric used to evaluate our generated binarized MNIST samples strictly follows the codebase provided in the SFM~\cite{cheng2024categorical} repository. We visualize the generated high-quality samples on the left side of Figure~\ref{fig:mnist_performance}.

\subsection{Ablation Study on Bernoulli Probability Path Scheduling Strategies}
Fig. \ref{fig:noise_scheduler_paths} visualizes how, under different probability paths, the probabilities associated with 1 (blue curves) and 0 (red curves) are progressively driven towards 0.5 along the diffusion trajectory. To systematically assess the impact of these schedulers on model performance, we select 3 different probability path schedulers and conduct experiments on the LSUN Churches 256×256 dataset using the official BLD codebase. We keep the network architecture, learning rate, loss function, and all optimization hyperparameters exactly the same, and train with a batch size of 96. The only difference lies in the algorithm at the training and sampling stages.

\begin{figure*}[htbp]
    \centering
    \begin{subfigure}{0.30\textwidth}  
        \centering
        \includegraphics[width=\linewidth]{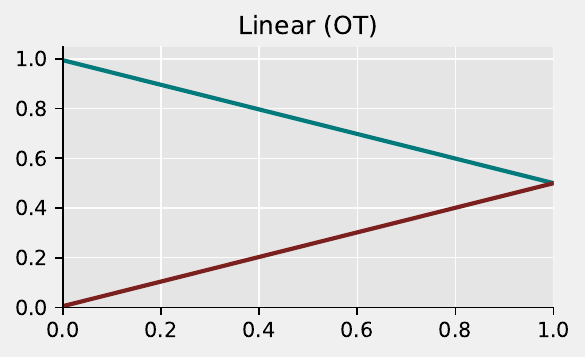}
        \caption{Linear}
        \label{fig:linear}
    \end{subfigure}
    \hfill
    \begin{subfigure}{0.30\textwidth}  
        \centering
        \includegraphics[width=\linewidth]{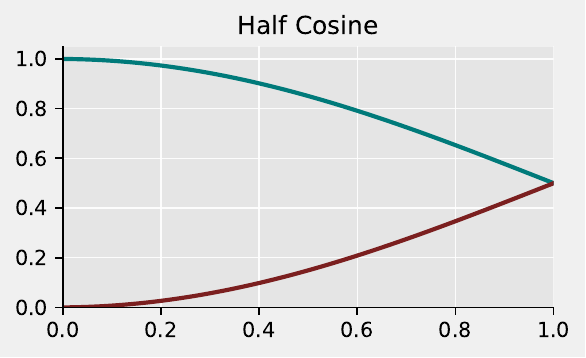}
        \caption{Half cos}
        \label{fig:Half cosine}
    \end{subfigure}
    \hfill
    \begin{subfigure}{0.30\textwidth}  
        \centering
        \includegraphics[width=\linewidth]{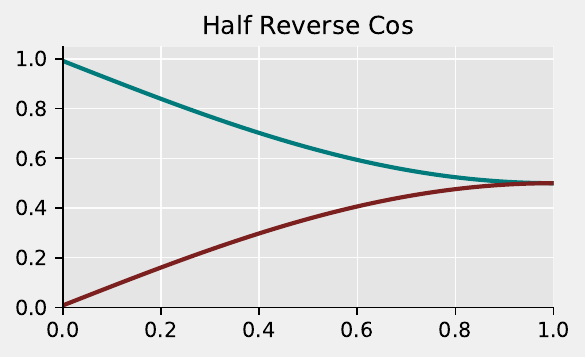}
        \caption{Reverse half cos}
        \label{fig:half_reverse_cosine}
    \end{subfigure}
    \hfill

    \caption{Bernoulli probability path schedulers.}
    \label{fig:noise_scheduler_paths}
\end{figure*}

\begin{table}[htbp]
    \centering
    \caption{FID comparison of different Bernoulli probability path schedulers for BFM and BLD on LSUN Churches 256$\times$256, evaluated at 100K training iterations.}
    \label{tab:fid_results}
    \begin{tabular}{lcc}
        \toprule
        \textbf{Scheduler} & \multicolumn{2}{c}{\textbf{FID} $\downarrow$} \\
        \cmidrule(lr){2-3}
                           & \textbf{BFM} & \textbf{BLD} \\
        \midrule
        Linear             & \textbf{9.22}         & 9.55         \\
        Half-cosine        & 9.41         & 9.53         \\
        Reverse half-cos   & 9.27         & 9.54         \\
        \bottomrule
    \end{tabular}
\end{table}

We summarize the ablation results in Table~\ref{tab:fid_results}. 
Under all three probability paths (linear, half-cosine, and reverse half-cosine), BFM consistently achieves lower FID than BLD. Among them, the linear scheduler gives the best overall performance for BFM (9.22). We also note that the FID of BFM varies slightly more across different schedulers than that of BLD, although the overall variation remains small. This indicates that while the specific choice of scheduling function has a measurable impact on the final generation quality, BFM maintains stable performance across different probability paths. Consequently, given its superior empirical performance, we adopt the linear schedule as the default configuration for all subsequent BFM experiments.

\section{Discussion}
\label{discussion}
We presented Bernoulli Flow Models (BFM), a self-consistent generative framework that defines a continuous global Bernoulli probability path between data and noise and derives analytical posterior transitions over arbitrary time intervals. This construction decouples inference from a fixed discrete grid: when the NFE changes, BFM re-evaluates the posterior on the new intervals rather than approximating skipped one-step transitions. Experiments show robust training-free low-NFE sampling while remaining competitive under standard inference. Our evaluation is currently limited to representative image-generation and scientific binary benchmarks, leaving broader domain-specific validation for future work. We also plan to combine BFM with distillation or consistency-style training for very-few-step or one-step generation.

\acksection
This work is supported by the Science and Technology Development Fund of Macau Project (grant
No. 0096/2023/RIA2, 0044/2024/AGJ, 0140/2024/AGJ, 0084/2024/RIB2), and is partially supported
by the AI Supercomputing Center of Macau University of Science and Technology (MUST). We
acknowledge the technical support provided by the Information Technology Development Office of
MUST. We thank the anonymous reviewers for their constructive feedback and valuable suggestions. The authors declare no competing interests.

\clearpage  
{
\bibliographystyle{plain}  
\bibliography{ref}         
}

\clearpage  
\appendix

  \section{Proofs of Bernoulli Probability Flow Dynamics}
  \label{appendix_proofs}

  In this section, we provide the detailed and rigorous proofs for the Bernoulli Probability Flow Dynamics introduced in Section \ref{sec:method}. 

  \subsection{Mathematical Preliminaries}

  We first formalize the core algebraic operator used to simplify the binary probability space.

  \begin{definition}[Symmetric Binary Combination]
  We define the symmetric combination function $g: [0,1] \times [0,1] \to [0,1]$ as:
  \begin{equation}
      g(a,b) = a(1-b) + b(1-a).
  \end{equation}
  \end{definition}

  \subsection{Marginal Distribution Dynamics}

  We begin by establishing the marginal distribution from the data $X_{t_0}$ to an arbitrary state $X_{t_\tau}$.

  \begin{proposition}[Marginal Probability]
  \label{prop:marginal}
  Let $X_{t_0} \in \{0,1\}$ be a data sample and $X_1 \sim \mathcal{B}(0.5)$ be a uniform prior. For time parameters $t_\tau \in [0,1]$ evaluated along a decreasing schedule, let $X_{t_\tau}$ be defined conditionally as $X_{t_\tau} \sim \mathcal{B}\left( t_\tau X_{t_0} + (1 - t_\tau) X_1 \right)$. The exact marginal distribution is given by:
  \begin{equation}
      q(X_{t_\tau} \mid X_{t_0}) = \mathcal{B}\left(X_{t_\tau}; g\left(X_{t_{0}}, \frac{1-t_{\tau}}{2}\right)\right).
  \end{equation}
  \end{proposition}

  \begin{proof}
  We aim to derive the conditional distribution $q(X_{t_\tau} \mid X_{t_0})$ by marginalizing over the uniform prior $X_1 \sim \mathcal{B}(0.5)$. By definition, the Bernoulli parameter for the intermediate state is $p = t_\tau X_{t_0} + (1 - t_\tau) X_1$. Using the linearity of expectation, the marginal probability of observing state 1 is computed directly as:
  \begin{align*}
  \mathbb{P}(X_{t_\tau} = 1 \mid X_{t_0}) &= \mathbb{E}_{X_1} \left[ t_\tau X_{t_0} + (1 - t_\tau) X_1 \right] \\
  &= t_\tau X_{t_0} + (1 - t_\tau) \cdot \mathbb{E}[X_1] \\
  &= t_\tau X_{t_0} + 0.5(1 - t_\tau).
  \end{align*}
  Notice that this compact expression naturally evaluates to $0.5(1 - t_\tau)$ when $X_{t_0} = 0$, and to $0.5(1 + t_\tau)$ when $X_{t_0} = 1$. 

  To unify this parameter into a symmetric form, we algebraically rewrite the expression as:
  \begin{equation}
  \alpha_{t_\tau} := \mathbb{P}(X_{t_\tau} = 1 \mid X_{t_0}) = \frac{1 + (2X_{t_0} - 1)t_\tau}{2}.
  \label{eq:appendix_alpha_tau}
  \end{equation}
  It is straightforward to algebraically verify that $g\left(X_{t_{0}}, \frac{1-t_{\tau}}{2}\right) = X_{t_0}\left(1 - \frac{1-t_\tau}{2}\right) + \frac{1-t_\tau}{2}(1 - X_{t_0})$, which simplifies exactly to $\frac{1 + (2X_{t_0} - 1)t_\tau}{2}$. Therefore, the marginal probability is exactly parameterized by $\alpha_{t_\tau} = g\left(X_{t_{0}}, \frac{1-t_{\tau}}{2}\right)$.
  \end{proof}

  \subsection{One-Step Forward Transition Dynamics}

  We now prove the transition probability between two sequential states $X_{t_{\tau - 1}}$ and $X_{t_\tau}$.

  \begin{proposition}[Forward Transition Probability]
  \label{prop:forward}
  Let $0\leq t_\tau<t_{\tau-1}\leq1$, so that $t_{\tau-1}>0$. Assuming a Markovian transition process from $X_{t_{\tau-1}}$ to $X_{t_\tau}$ with an error rate $\gamma_{t_\tau}$, the transition probability is uniquely specified by:
  \begin{equation}
  q(X_{t_\tau} \mid X_{t_{\tau-1}}) = \mathcal{B}\left(X_{t_\tau}; g(X_{t_{\tau-1}}, \gamma_{t_\tau})\right), \quad \text{where} \quad \gamma_{t_{\tau}} = \frac{0.5(t_{\tau-1} - t_{\tau})}{t_{\tau-1}}.
  \end{equation}
  \end{proposition}

  \begin{proof}
  Let $\alpha_{t_\tau}$ and $\alpha_{t_{\tau-1}}$ be the marginal probabilities of state 1 at their respective time steps, defined via Eq.~\ref{eq:appendix_alpha_tau}. By the law of total probability, the marginal probability $\alpha_{t_\tau}$ must aggregate all paths from the previous state:
  \begin{equation}
  \alpha_{t_\tau} = \sum_{x \in \{0,1\}} q(X_{t_\tau} = 1 \mid X_{t_{\tau-1}} = x) \cdot q(X_{t_{\tau-1}} = x \mid X_{t_0}).
  \end{equation}
  Substituting the parameterized transition probabilities yields:
  \begin{equation}
  \alpha_{t_\tau} = \alpha_{t_{\tau-1}}(1 - \gamma_{t_\tau}) + (1 - \alpha_{t_{\tau-1}})\gamma_{t_\tau}.
  \end{equation}
  Rearranging the terms and solving for the flip probability $\gamma_{t_\tau}$:
  \begin{equation}
  \alpha_{t_\tau} = \alpha_{t_{\tau-1}} - 2\gamma_{t_\tau}\alpha_{t_{\tau-1}} + \gamma_{t_\tau} \implies \gamma_{t_\tau} = \frac{\alpha_{t_{\tau-1}} - \alpha_{t_\tau}}{2\alpha_{t_{\tau-1}} - 1}.
  \end{equation}
  Substituting $\alpha_{t_\tau} = \frac{1 + (2X_{t_0} - 1)t_\tau}{2}$ and $\alpha_{t_{\tau-1}} = \frac{1 + (2X_{t_0} - 1)t_{\tau-1}}{2}$ into the equation exactly cancels out the data-dependent term $(2X_{t_0} - 1)$, resulting in:
  \begin{equation}
  \gamma_{t_{\tau}} = \frac{\frac{1}{2}(2X_{t_0} - 1)(t_{\tau-1} - t_\tau)}{(2X_{t_0} - 1)t_{\tau-1}} = \frac{0.5(t_{\tau-1} - t_{\tau})}{t_{\tau-1}}.
  \label{eq:t_gamma_t_tau}
  \end{equation}
  This confirms the exact functional form of the forward transition.
  \end{proof}

  \subsection{Analytical One-Step Posterior Derivation}

  \begin{proposition}[One-Step Closed-Form Posterior]
  \label{prop:one_step_posterior}
  The ground-truth posterior distribution of $X_{t_{\tau-1}}$ conditioned on $X_{t_\tau}$ and the source data $X_{t_0}$ admits the exact closed-form solution:
  \begin{equation}
      q(X_{t_{\tau-1}} \mid X_{t_{\tau}}, X_{t_{0}}) =  \mathcal{B}\left(X_{t_{\tau-1}}; \frac{g(X_{t_{\tau}},\gamma_{t_{\tau}})\alpha_{t_{\tau-1}}}{g(g(X_{t_{\tau}},\gamma_{t_{\tau}}),1-\alpha_{t_{\tau-1}})}\right).
  \end{equation}
  \end{proposition}

  \begin{proof}
  Our objective is to find the analytical Bernoulli parameter for the reverse transition, which requires calculating the probability of $X_{t_{\tau-1}} = 1$ using Bayes' theorem:
  \begin{equation}
  q(X_{t_{\tau-1}}=1 \mid X_{t_\tau}, X_{t_0}) = 
  \frac{q(X_{t_\tau} \mid X_{t_{\tau-1}} = 1) \cdot q(X_{t_{\tau-1}} = 1 \mid X_{t_0})}{\sum_{x \in \{0,1\}} q(X_{t_\tau} \mid X_{t_{\tau-1}} = x) \cdot q(X_{t_{\tau-1}} = x \mid X_{t_0})}.
  \label{eq:appendix_bayes_onestep}
  \end{equation}
  Based on the dynamics established in Propositions \ref{prop:marginal} and \ref{prop:forward}, we define the constituent transition probabilities:
  \begin{equation}
  \begin{aligned}
      q(X_{t_\tau} \mid X_{t_{\tau-1}} = 1) &= g(X_{t_\tau}, \gamma_{t_\tau}), \quad & q(X_{t_{\tau-1}} = 1 \mid X_{t_0}) &= \alpha_{t_{\tau-1}}, \\
      q(X_{t_\tau} \mid X_{t_{\tau-1}} = 0) &= 1 - g(X_{t_\tau}, \gamma_{t_\tau}), \quad & q(X_{t_{\tau-1}} = 0 \mid X_{t_0}) &= 1 - \alpha_{t_{\tau-1}}.
  \end{aligned}
  \end{equation}
  Substituting these precise terms into the fractional Bayes' formulation in Eq.~\ref{eq:appendix_bayes_onestep}, we obtain:
  \begin{equation}
  q(X_{t_{\tau-1}} = 1 \mid X_{t_\tau}, X_{t_0}) = \frac{g(X_{t_{\tau}},\gamma_{t_{\tau}}) \alpha_{t_{\tau-1}}}{g(X_{t_{\tau}},\gamma_{t_{\tau}}) \alpha_{t_{\tau-1}} + \left(1 - g(X_{t_{\tau}},\gamma_{t_{\tau}})\right)(1 - \alpha_{t_{\tau-1}})}.
  \end{equation}
  By recognizing the symmetric structure of the denominator, we apply the definition $g(a,b) = a(1-b) + b(1-a)$, which folds the denominator perfectly into $g\left(g(X_{t_{\tau}},\gamma_{t_{\tau}}),1-\alpha_{t_{\tau-1}}\right)$. Thus, the final expression simplifies to:
  \begin{equation}
  q(X_{t_{\tau-1}}=1 \mid X_{t_\tau}, X_{t_0}) = \frac{g(X_{t_{\tau}},\gamma_{t_{\tau}})\alpha_{t_{\tau-1}}}{g(g(X_{t_{\tau}},\gamma_{t_{\tau}}),1-\alpha_{t_{\tau-1}})},
  \end{equation}
  concluding the derivation of the one-step posterior.
  \end{proof}

  \subsection{Exact Cross-Step Posterior (Proof of Theorem \ref{theorem:cross_step})}

  Finally, we extend the one-step derivation to any arbitrary interval to prove the fast-sampling mechanism introduced in the main text. 

  \textbf{Theorem \ref{theorem:cross_step} (Exact Cross-Step Posterior, restated).} \textit{For any reverse interval satisfying $0\leq t_\tau<t_{\tau-k}\leq1$ (and hence $t_{\tau-k}>0$), where $k>1$, the exact cross-step posterior of the Bernoulli Flow Model admits a perfect closed-form solution:}
  \begin{equation}
      q(X_{t_{\tau-k}} \mid X_{t_{\tau}}, X_{t_{0}}) =  \mathcal{B}\left(X_{t_{\tau-k}}; \frac{g(X_{t_{\tau}},\gamma_{t_{\tau} | t_{\tau-k}})\alpha_{t_{\tau-k}}}{g(g(X_{t_{\tau}},\gamma_{t_{\tau} | t_{\tau-k}}),1-\alpha_{t_{\tau-k}})}\right).
  \end{equation}

  \begin{proof}
  Because our parameters $\alpha$ and $\gamma$ are defined continuously across the interpolation path, the algebraic relationships proven in Propositions \ref{prop:marginal} and \ref{prop:forward} hold for any arbitrary temporal stride. To jump from $t_\tau$ directly back to $t_{\tau-k}$, we generalize the transition error rate over this specific interval as $\gamma_{t_{\tau} | t_{\tau-k}} = \frac{0.5(t_{\tau-k} - t_{\tau})}{t_{\tau-k}}$.

  Using Bayes' theorem, the probability of $X_{t_{\tau-k}} = 1$ is calculated exactly as:
  \begin{equation}
  q(X_{t_{\tau-k}}=1 \mid X_{t_\tau}, X_{t_0}) = 
  \frac{q(X_{t_\tau} \mid X_{t_{\tau-k}} = 1) \cdot q(X_{t_{\tau-k}} = 1 \mid X_{t_0})}{\sum_{x \in \{0,1\}} q(X_{t_\tau} \mid X_{t_{\tau-k}} = x) \cdot q(X_{t_{\tau-k}} = x \mid X_{t_0})}.
  \label{eq:appendix_bayes_crossstep}
  \end{equation}
  Substituting the generalized constituent parameters for the $k$-step stride:
  \begin{equation}
  \begin{aligned}
      q(X_{t_\tau} \mid X_{t_{\tau-k}} = 1) &= g(X_{t_\tau}, \gamma_{t_\tau | t_{\tau-k}}), \quad & q(X_{t_{\tau-k}} = 1 \mid X_{t_0}) &= \alpha_{t_{\tau-k}}, \\
      q(X_{t_\tau} \mid X_{t_{\tau-k}} = 0) &= 1 - g(X_{t_\tau}, \gamma_{t_\tau | t_{\tau-k}}), \quad & q(X_{t_{\tau-k}} = 0 \mid X_{t_0}) &= 1 - \alpha_{t_{\tau-k}}.
  \end{aligned}
  \end{equation}
  Inserting these into Eq.~\ref{eq:appendix_bayes_crossstep} yields the expanded fraction:
  \begin{equation}
  q(X_{t_{\tau-k}} = 1 \mid X_{t_\tau}, X_{t_0}) = \frac{g(X_{t_{\tau}},\gamma_{t_{\tau} | t_{\tau-k}}) \alpha_{t_{\tau-k}}}{g(X_{t_{\tau}},\gamma_{t_{\tau} | t_{\tau-k}}) \alpha_{t_{\tau-k}} + \left(1 - g(X_{t_{\tau}},\gamma_{t_{\tau} | t_{\tau-k}})\right)(1 - \alpha_{t_{\tau-k}})}.
  \end{equation}
  Applying the symmetric operator definition $g(a,b) = a(1-b) + b(1-a)$ to the denominator precisely folds it into $g\left(g(X_{t_{\tau}},\gamma_{t_{\tau} | t_{\tau-k}}),1-\alpha_{t_{\tau-k}}\right)$. Consequently, the analytical cross-step posterior is cleanly obtained:
  \begin{equation}
  q(X_{t_{\tau-k}}=1 \mid X_{t_\tau}, X_{t_0}) = \frac{g(X_{t_{\tau}},\gamma_{t_{\tau} | t_{\tau-k}})\alpha_{t_{\tau-k}}}{g(g(X_{t_{\tau}},\gamma_{t_{\tau} | t_{\tau-k}}),1-\alpha_{t_{\tau-k}})},
  \end{equation}
  which establishes the structural consistency of the oracle posterior under changes of the inference grid. The learned sampler remains subject to the prediction error analyzed in the following subsection.
  \end{proof}

\subsection{Error Bounds for Learned Reverse Sampling}
\label{app:learned_error_bounds}

We provide the proofs of Propositions~
\ref{prop:learned_error} and \ref{prop:accumulated_error}.
We first establish the coordinate-wise transition bounds,
then extend them to the full binary state and the
accumulated sampling error.

\subsubsection{Prediction-Induced TV Bound}

Consider an arbitrary reverse interval
$0\leq t<s<1$, where $s=t_{\tau-k}$ and
$t=t_\tau$. Following Eq.~\ref{eq:cross_step_gamma},
the transition flip probability is
\begin{equation}
\gamma_{t\mid s}=\frac{s-t}{2s}.
\end{equation}

For a fixed noisy state $X_t$, we express the
analytical posterior parameter as a function of
the clean-state input $u\in[0,1]$:
\begin{equation}
r(u)=
\frac{
g(X_t,\gamma_{t\mid s})
g\left(u,\frac{1-s}{2}\right)
}{
g\left(
g(X_t,\gamma_{t\mid s}),
1-g\left(u,\frac{1-s}{2}\right)
\right)
}.
\label{eq:app_posterior_function}
\end{equation}

The oracle and learned posterior parameters are
therefore
\begin{equation}
q(X_s=1\mid X_t,X_{t_0})
 =r(X_{t_0}),
\qquad
p_\theta(X_s=1\mid X_t)
 =r(f_\theta(X_t,t)).
\end{equation}

Using the definition of $g$ in Eq.~(2),
we have
\begin{equation}
\frac{d}{du}
g\left(u,\frac{1-s}{2}\right)=s.
\end{equation}

Differentiating Eq.~\eqref{eq:app_posterior_function}
with respect to $u$ gives
\begin{equation}
r'(u)=
\frac{
s\,g(X_t,\gamma_{t\mid s})
\left[1-g(X_t,\gamma_{t\mid s})\right]
}{
\left[
g\left(
g(X_t,\gamma_{t\mid s}),
1-g\left(u,\frac{1-s}{2}\right)
\right)
\right]^2
}.
\label{eq:app_posterior_derivative}
\end{equation}

Since $X_t\in\{0,1\}$,
\begin{equation}
g(X_t,\gamma_{t\mid s})
\in
\left\{
\frac{s-t}{2s},
\frac{s+t}{2s}
\right\}.
\end{equation}

Consequently,
\begin{equation}
g(X_t,\gamma_{t\mid s})
\left[1-g(X_t,\gamma_{t\mid s})\right]
=\frac{s^2-t^2}{4s^2}.
\end{equation}

Moreover, the denominator before squaring in
Eq.~\eqref{eq:app_posterior_derivative}
can be expressed as
\begin{equation}
\begin{aligned}
g\left(
g(X_t,\gamma_{t\mid s}),
1-g\left(u,\frac{1-s}{2}\right)
\right)
&=
\begin{cases}
\frac{1-t}{2}+tu, & X_t=1,\\[2pt]
\frac{1+t}{2}-tu, & X_t=0.
\end{cases}
\end{aligned}
\end{equation}

For all $u\in[0,1]$, this quantity is bounded
below by $(1-t)/2$. Therefore,
\begin{equation}
|r'(u)|
\leq
\frac{s^2-t^2}{s(1-t)^2}.
\label{eq:app_lipschitz_bound}
\end{equation}

Applying the mean-value theorem yields
\begin{equation}
\left|
r(f_\theta(X_t,t))-r(X_{t_0})
\right|
\leq
\frac{s^2-t^2}{s(1-t)^2}
\left|f_\theta(X_t,t)-X_{t_0}\right|.
\end{equation}

For two Bernoulli distributions, the TV distance
equals the absolute difference between their
Bernoulli parameters. Hence,
\begin{equation}
D_{\mathrm{TV}}\big(
q(X_s\mid X_t,X_{t_0}),
p_\theta(X_s\mid X_t)
\big)
\leq
\frac{s^2-t^2}{s(1-t)^2}
\left|f_\theta(X_t,t)-X_{t_0}\right|.
\end{equation}

This proves Proposition~\ref{prop:learned_error}
for both adjacent-step and arbitrary cross-step
reverse transitions.

\subsubsection{KL Bound and Connection to BCE}

We next connect the prediction-induced
reverse-transition error to the BCE objective.

For $0\leq t<s<1$, the posterior parameter
in Eq.~\eqref{eq:app_posterior_function}
satisfies
\begin{equation}
r(u)\in[\delta_{s,t},1-\delta_{s,t}],
\qquad u\in[0,1],
\end{equation}
where
\begin{equation}
\delta_{s,t}
=
\frac{(s-t)(1-s)}{2s(1+t)}>0.
\label{eq:app_delta}
\end{equation}

To see this, both
$g(X_t,\gamma_{t\mid s})$ and
$1-g(X_t,\gamma_{t\mid s})$
are at least $(s-t)/(2s)$.
Similarly, both
$g(u,(1-s)/2)$ and
$1-g(u,(1-s)/2)$
are at least $(1-s)/2$,
while the common denominator in
Eq.~\eqref{eq:app_posterior_function}
is at most $(1+t)/2$.

For Bernoulli distributions with parameters
$a,b\in(0,1)$, the KL divergence satisfies
\begin{equation}
D_{\mathrm{KL}}
\big(\mathcal{B}(a)\|
\mathcal{B}(b)\big)
\leq
\frac{(a-b)^2}{b(1-b)}.
\end{equation}

Combining this inequality with
Eq.~\eqref{eq:app_lipschitz_bound} gives
\begin{equation}
D_{\mathrm{KL}}\big(
q(X_s\mid X_t,X_{t_0})
\|p_\theta(X_s\mid X_t)
\big)
\leq
C_{s,t}
\left|f_\theta(X_t,t)-X_{t_0}\right|^2,
\end{equation}
where
\begin{equation}
C_{s,t}
=
\frac{1}{\delta_{s,t}(1-\delta_{s,t})}
\left(
\frac{s^2-t^2}{s(1-t)^2}
\right)^2.
\label{eq:app_kl_constant}
\end{equation}

For any binary $x\in\{0,1\}$ and
$r\in(0,1)$,
\begin{equation}
|r-x|^2
\leq
-x\log r-(1-x)\log(1-r).
\label{eq:app_bce_inequality}
\end{equation}

Therefore, the coordinate-wise KL divergence
is bounded by
\begin{equation}
\begin{aligned}
D_{\mathrm{KL}}\big(
q(X_s\mid X_t,X_{t_0})
\|p_\theta(X_s\mid X_t)
\big)
&\leq C_{s,t}
\Big[
-X_{t_0}\log f_\theta(X_t,t)\\
&\quad-
(1-X_{t_0})
\log(1-f_\theta(X_t,t))
\Big].
\end{aligned}
\label{eq:app_bce_kl_bound}
\end{equation}

For the $d$-dimensional extension, assume that the
forward corruption acts independently across coordinates
conditional on the preceding state and that the learned
reverse transition is product Bernoulli conditional on
its current state. Under this assumption, the oracle
posterior conditioned on $X_{t_0}$ and the learned
transition both factorize coordinate-wise.
The full-state KL divergence is therefore
bounded by the sum of the corresponding
coordinate-wise BCE terms.

The true reverse transition is obtained by
marginalizing the oracle posterior over
the clean state:
\begin{equation}
q(X_s\mid X_t)
=
\mathbb{E}_{q(X_{t_0}\mid X_t)}
\left[
q(X_s\mid X_t,X_{t_0})
\right].
\end{equation}

By convexity of KL divergence,
\begin{equation}
D_{\mathrm{KL}}\big(
q(X_s\mid X_t)
\|p_\theta(X_s\mid X_t)
\big)
\leq
\mathbb{E}_{q(X_{t_0}\mid X_t)}
\Big[
D_{\mathrm{KL}}\big(
q(X_s\mid X_t,X_{t_0})
\|p_\theta(X_s\mid X_t)
\big)
\Big].
\end{equation}

Combining this with
Eq.~\eqref{eq:app_bce_kl_bound}
shows that, for each fixed non-terminal interval
$(s,t)$, the discrepancy from the true reverse
transition is controlled by the expected BCE terms
with the interval-dependent factor $C_{s,t}$.
For any fixed finite collection of such intervals,
the maximum of $C_{s,t}$ over that collection is
finite, so summing over coordinates and averaging
over the corresponding training distribution gives
an upper bound proportional to the BCE objective in
Eq.~\eqref{eq:bce_loss}, restricted to those intervals.
This is an interval-wise
control statement; no grid-uniform constant is
claimed as $s\uparrow1$.

\subsubsection{Accumulated Low-NFE Sampling Error}

Consider an arbitrary inference grid
\begin{equation}
1=t_0>t_1>\cdots>t_K=0,
\end{equation}
where $K$ is the number of reverse transitions.

For each intermediate step $j\geq2$, assume
the expected soft prediction error satisfies
\begin{equation}
\mathbb{E}_{q(X_{t_0},X_{t_j})}
\Big[
\left\|f_\theta(X_{t_j},t_j)-X_{t_0}\right\|_1
\Big]
\leq\varepsilon_j.
\label{eq:app_expected_error}
\end{equation}

Under the same coordinate-wise product structure, the
full-state TV distance is bounded by the
sum of the coordinate-wise TV distances.
Combining this property with
Proposition~\ref{prop:learned_error}
and the convexity of TV gives
\begin{equation}
\mathbb{E}_{q(X_{t_j})}
\Big[
D_{\mathrm{TV}}\big(
q(X_{t_{j-1}}\mid X_{t_j}),
p_\theta(X_{t_{j-1}}\mid X_{t_j})
\big)
\Big]
\leq
\frac{t_{j-1}^2-t_j^2}
{t_{j-1}(1-t_j)^2}
\varepsilon_j.
\label{eq:app_step_tv_error}
\end{equation}

The final transition to $t_0=1$ is treated
separately.\par
\noindent We define its expected
reverse-transition discrepancy as
\begin{equation}
E_{\mathrm{rec}}
:=
\mathbb{E}_{q(X_{t_1})}
\Big[
D_{\mathrm{TV}}\big(
q(X_{t_0}\mid X_{t_1}),
p_\theta(X_{t_0}\mid X_{t_1})
\big)
\Big].
\label{eq:app_reconstruction_error}
\end{equation}

By the contraction property of Markov kernels
and the triangle inequality, each reverse
step satisfies
\begin{equation}
\begin{aligned}
D_{\mathrm{TV}}\big(
q(X_{t_{j-1}}),
p_\theta(X_{t_{j-1}})
\big)
&\leq
D_{\mathrm{TV}}\big(
q(X_{t_j}),p_\theta(X_{t_j})
\big)\\
&\mathrel{+}
\mathbb{E}_{q(X_{t_j})}
\Big[
D_{\mathrm{TV}}\big(
q(X_{t_{j-1}}\mid X_{t_j}),
p_\theta(X_{t_{j-1}}\mid X_{t_j})
\big)
\Big].
\end{aligned}
\label{eq:app_tv_telescope}
\end{equation}

Since the true and learned reverse processes
start from the same Bernoulli prior,
\begin{equation}
q(X_{t_K})=p_\theta(X_{t_K})=\pi_1,
\end{equation}
the initial TV discrepancy is zero.

Applying Eq.~\eqref{eq:app_tv_telescope}
recursively and using
Eq.~\eqref{eq:app_step_tv_error}
and Eq.~\eqref{eq:app_reconstruction_error}
yields
\begin{equation}
 D_{\mathrm{TV}}\big(
\pi_0,p_\theta(X_{t_0})
\big)
\leq
\min\left\{
1,\,
E_{\mathrm{rec}}+
\sum_{j=2}^{K}
\frac{t_{j-1}^{2}-t_j^{2}}
{t_{j-1}(1-t_j)^2}
\varepsilon_j
\right\}.
\label{eq:app_accumulated_tv}
\end{equation}

This proves Proposition~\ref{prop:accumulated_error}.

The result shows that the oracle cross-step
posterior remains analytically exact, while
the learned sampler accumulates errors
determined by prediction accuracy and
the selected inference intervals.

\subsection{Connecting the Analytical Posterior to the Practical Sampler}
\label{app:oracle-to-learned}

This subsection clarifies how the clean-data-conditioned posterior in
Theorem~\ref{theorem:cross_step} is related to the practical reverse
transition used during generation. Specifically, we explain three points:
why the posterior conditioned on $X_{t_0}$ is available as a
ground-truth denoising target during training, why likelihood-based
training requires the learned reverse transition to approximate this
posterior, and why the resulting reverse transition can be parameterized
by predicting $X_{t_0}$ rather than directly predicting the complete
posterior probability.

\paragraph{Forward and reverse Bernoulli processes.}
Let
\begin{equation}
1=t_0>t_1>\cdots>t_T=0
\end{equation}
denote a finite time grid. Here, $X_{t_0}$ is a clean binary observation
drawn from the data distribution $\pi_0$, whereas $X_{t_T}$ follows the
product-Bernoulli prior
$\pi_1=\mathcal{B}(\tfrac{1}{2}\mathbf{1}_d)
=\bigotimes_{i=1}^d\mathcal{B}(\tfrac{1}{2})$.

The predefined forward process progressively transforms the clean data
into Bernoulli noise:
\begin{equation}
q(X_{t_{1:T}}\mid X_{t_0})
=
\prod_{\tau=1}^{T}
q(X_{t_\tau}\mid X_{t_{\tau-1}}).
\label{eq:forward-joint}
\end{equation}
All transitions in this process are analytically specified and contain
no learnable parameters.

The generative model operates in the opposite direction:
\begin{equation}
p_\theta(X_{t_{0:T}})
=
p(X_{t_T})
\prod_{\tau=1}^{T}
p_\theta(X_{t_{\tau-1}}\mid X_{t_\tau}).
\label{eq:reverse-joint}
\end{equation}
Generation begins from $X_{t_T}\sim\pi_1$ and repeatedly
samples from the learned reverse transitions until a clean sample
$X_{t_0}$ is obtained. The learning problem is therefore to construct
$p_\theta(X_{t_{\tau-1}}\mid X_{t_\tau})$ using only the current noisy
state $X_{t_\tau}$.

\paragraph{Why the oracle posterior appears during training.}
Ideally, the reverse sampler should reproduce the true reverse transition
\begin{equation}
q(X_{t_{\tau-1}}\mid X_{t_\tau}).
\end{equation}
This distribution is generally intractable because evaluating it requires
marginalizing over the unknown data distribution. During training,
however, the clean sample $X_{t_0}$ is observed from the dataset, and
$X_{t_\tau}$ is generated by applying the predefined forward process to
this known sample. We can therefore evaluate the tractable conditional
posterior
\begin{equation}
q(X_{t_{\tau-1}}\mid X_{t_\tau},X_{t_0}).
\end{equation}

We refer to this distribution as the oracle posterior because it has
access to the clean state toward which $X_{t_\tau}$ should be denoised.
The conditioning variable $X_{t_0}$ is available only when constructing
the training target; it is not assumed to be available during
generation.

The role of this oracle posterior follows directly from
maximum-likelihood training. The objective of the generative model is to
maximize the likelihood assigned to the observed clean data,
$\log p_\theta(X_{t_0})$. Marginalizing over the intermediate binary
states gives
\begin{align}
\log p_\theta(X_{t_0})
&=
\log
\sum_{X_{t_1},\ldots,X_{t_T}}
p_\theta(X_{t_0},X_{t_1},\ldots,X_{t_T})
\nonumber\\
&=
\log
\sum_{X_{t_1},\ldots,X_{t_T}}
p_\theta(X_{t_{0:T}})
\frac{
q(X_{t_{1:T}}\mid X_{t_0})
}{
q(X_{t_{1:T}}\mid X_{t_0})
}
\nonumber\\
&=
\log
\sum_{X_{t_1},\ldots,X_{t_T}}
q(X_{t_{1:T}}\mid X_{t_0})
\frac{
p_\theta(X_{t_{0:T}})
}{
q(X_{t_{1:T}}\mid X_{t_0})
}
\nonumber\\
&=
\log
\mathbb{E}_{q(X_{t_{1:T}}\mid X_{t_0})}
\left[
\frac{
p_\theta(X_{t_{0:T}})
}{
q(X_{t_{1:T}}\mid X_{t_0})
}
\right].
\label{eq:data-likelihood-expectation}
\end{align}
Applying Jensen's inequality produces a tractable lower bound:
\begin{equation}
\log p_\theta(X_{t_0})
\geq
\mathbb{E}_{q}
\left[
\log
\frac{
p_\theta(X_{t_{0:T}})
}{
q(X_{t_{1:T}}\mid X_{t_0})
}
\right].
\end{equation}

Substituting the forward factorization in
Eq.~\eqref{eq:forward-joint} and the reverse factorization in
Eq.~\eqref{eq:reverse-joint}, and rearranging the forward transitions
using Bayes' rule, decomposes this lower bound as
\begin{align}
\log p_\theta(X_{t_0})
\geq{}&
\mathbb{E}_{q}
\left[
\log p_\theta(X_{t_0}\mid X_{t_1})
\right]
\nonumber\\
&-
D_{\mathrm{KL}}
\left(
q(X_{t_T}\mid X_{t_0})
\Vert
p(X_{t_T})
\right)
\nonumber\\
&-
\sum_{\tau=2}^{T}
\mathbb{E}_{q(X_{t_\tau}\mid X_{t_0})}
\left[
D_{\mathrm{KL}}
\left(
q(X_{t_{\tau-1}}\mid X_{t_\tau},X_{t_0})
\Vert
p_\theta(X_{t_{\tau-1}}\mid X_{t_\tau})
\right)
\right].
\label{eq:bernoulli-elbo}
\end{align}

Equation~\eqref{eq:bernoulli-elbo} makes the learning objective explicit.
The first term encourages reconstruction of the clean sample from the
least corrupted state. The second term matches the endpoint of the
forward process to the prescribed Bernoulli prior and contains no
learnable reverse-transition parameters. Most importantly, maximizing
the lower bound requires minimizing the denoising-matching terms
\begin{equation}
D_{\mathrm{KL}}
\left(
q(X_{t_{\tau-1}}\mid X_{t_\tau},X_{t_0})
\Vert
p_\theta(X_{t_{\tau-1}}\mid X_{t_\tau})
\right).
\label{eq:posterior-matching-objective}
\end{equation}

Thus, the oracle posterior is not an additional assumption made at
inference time. It is the tractable ground-truth reverse transition that
the practical sampler is trained to approximate. By making these
per-step reverse transitions accurate, the learned process can
iteratively transform samples from the Bernoulli prior back toward the
data distribution.
\paragraph{Why predicting $X_{t_0}$ is sufficient.}
Equation~\eqref{eq:posterior-matching-objective} specifies the learning
target: the practical reverse transition
$p_\theta(X_{t_{\tau-k}}\mid X_{t_\tau})$ should approximate the
ground-truth posterior
$q(X_{t_{\tau-k}}\mid X_{t_\tau},X_{t_0})$. It does not, however,
specify how $p_\theta$ should be parameterized.

One possible way to parameterize the practical reverse transition is to
use a neural network to directly predict the complete Bernoulli
probability of $p_\theta(X_{t_{\tau-k}}\mid X_{t_\tau})$ and train it to match the
Bernoulli parameter of the ground-truth posterior
$q(X_{t_{\tau-k}}\mid X_{t_\tau},X_{t_0})$. In other words, the network
could directly regress the full oracle posterior probability for each
noisy state and time interval. However, as noted in
BLD~\cite{wang2023binary}, directly modeling such reverse-transition
probabilities can be challenging, because the model must accurately
regress sophisticated interpolations that vary jointly with the current
noisy state and the noise level.

In BFM, Theorem~\ref{theorem:cross_step} provides the analytical form
of this posterior. For an arbitrary reverse interval from $t_\tau$ to
$t_{\tau-k}$, the oracle posterior is
\begin{equation}
q(X_{t_{\tau-k}}\mid X_{t_\tau},X_{t_0})
=
\mathcal{B}
\left(
X_{t_{\tau-k}};
\frac{
g(X_{t_\tau},\gamma_{t_\tau\mid t_{\tau-k}})
\alpha_{t_{\tau-k}}
}{
g\left(
g(X_{t_\tau},\gamma_{t_\tau\mid t_{\tau-k}}),
1-\alpha_{t_{\tau-k}}
\right)
}
\right),
\label{eq:oracle-cross-step-relation}
\end{equation}
where
\begin{equation}
\alpha_{t_{\tau-k}}
=
g\left(
X_{t_0},
\frac{1-t_{\tau-k}}{2}
\right),
\qquad
\gamma_{t_\tau\mid t_{\tau-k}}
=
\frac{0.5(t_{\tau-k}-t_\tau)}{t_{\tau-k}}.
\label{eq:oracle-known-quantities}
\end{equation}

\begingroup
\raggedbottom
For a given reverse update, $X_{t_\tau}$ and the inference times are
known, and $\gamma_{t_\tau\mid t_{\tau-k}}$ is analytically determined.
Thus, the only unknown quantity in
Eq.~\eqref{eq:oracle-cross-step-relation} is $X_{t_0}$.
We use the neural network $f_\theta(X_{t_\tau},t_\tau)$ to predict
$X_{t_0}$. Hence, the posterior $p_\theta$ parameterized by $f_\theta$ can be
written as
\begin{align}
&p_\theta(X_{t_{\tau-k}}\mid X_{t_\tau})
\nonumber\\
&=
\mathcal{B}
\left(
X_{t_{\tau-k}};
\frac{
g(X_{t_\tau},\gamma_{t_\tau\mid t_{\tau-k}})
g\left(
f_\theta(X_{t_\tau},t_\tau),
\frac{1-t_{\tau-k}}{2}
\right)
}{
g\left(
g(X_{t_\tau},\gamma_{t_\tau\mid t_{\tau-k}}),
1-
g\left(
f_\theta(X_{t_\tau},t_\tau),
\frac{1-t_{\tau-k}}{2}
\right)
\right)
}
\right).
\label{eq:learned-cross-step-relation}
\end{align}

Therefore, $X_{t_0}$-prediction provides a structured parameterization
of posterior matching. In the ideal case where
$f_\theta(X_{t_\tau},t_\tau)=X_{t_0}$,
Eq.~\eqref{eq:learned-cross-step-relation} exactly recovers the oracle
posterior in Eq.~\eqref{eq:oracle-cross-step-relation}. This
reparameterization through $X_{t_0}$-prediction corresponds to the
standard $x_0$-prediction strategy widely used in DDPMs and SDE diffusion models
~\cite{ho2020denoising,song2020score}.

\section{Visualization of Different Temperatures of BFM Generated Samples}
\label{appendix_temperature_visualization}

\begin{figure}[H]
    \centering
    \includegraphics[width=0.8\textwidth]{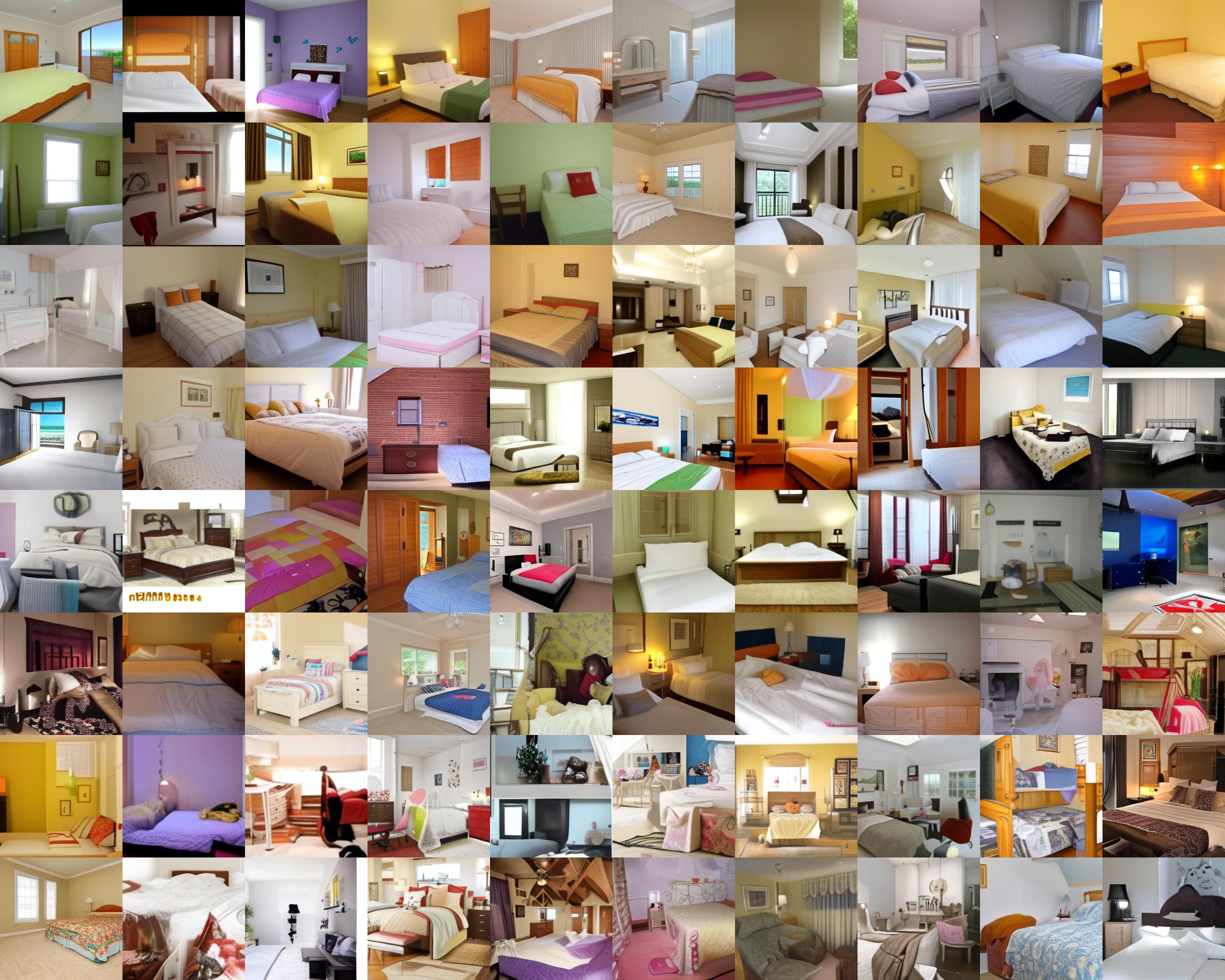}
    \caption{Unconditional $256 \times 256$ samples on LSUN Bedrooms. Sampling temperatures $\kappa$  are linearly interpolated from 0.5 to 1.0 (left to right).}
    \label{fig:big_bedroom}
\end{figure}
\clearpage
\endgroup


\begin{figure}[H]
    \centering
    \includegraphics[width=0.8\textwidth]{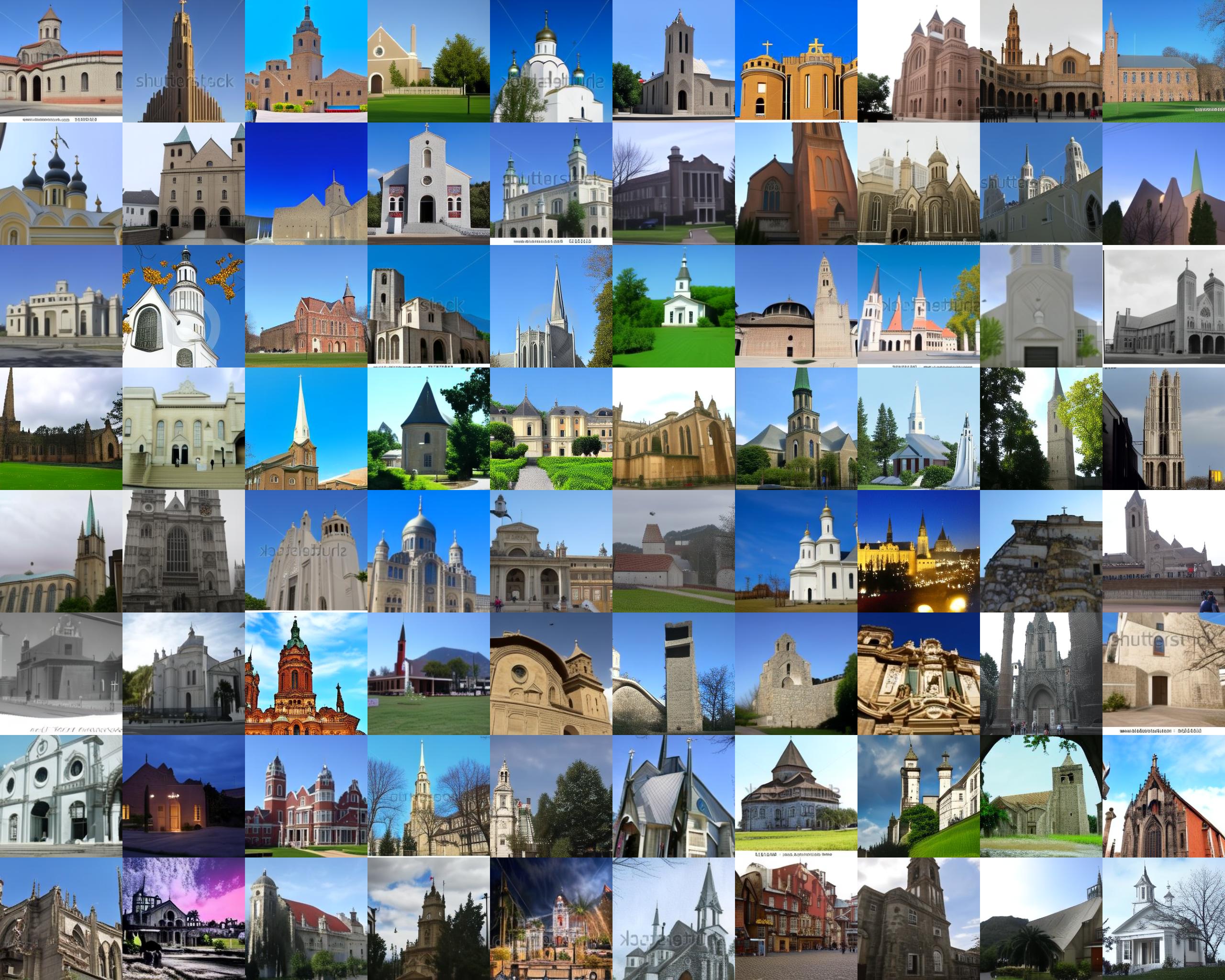}
    \caption{Unconditional $256 \times 256$ samples on LSUN Churches. Sampling temperatures $\kappa$ are linearly interpolated from 0.5 to 1.0 (left to right).}
    \label{fig:big_church}
\end{figure}

\begin{figure}[H]
    \centering
    \includegraphics[width=0.8\textwidth]{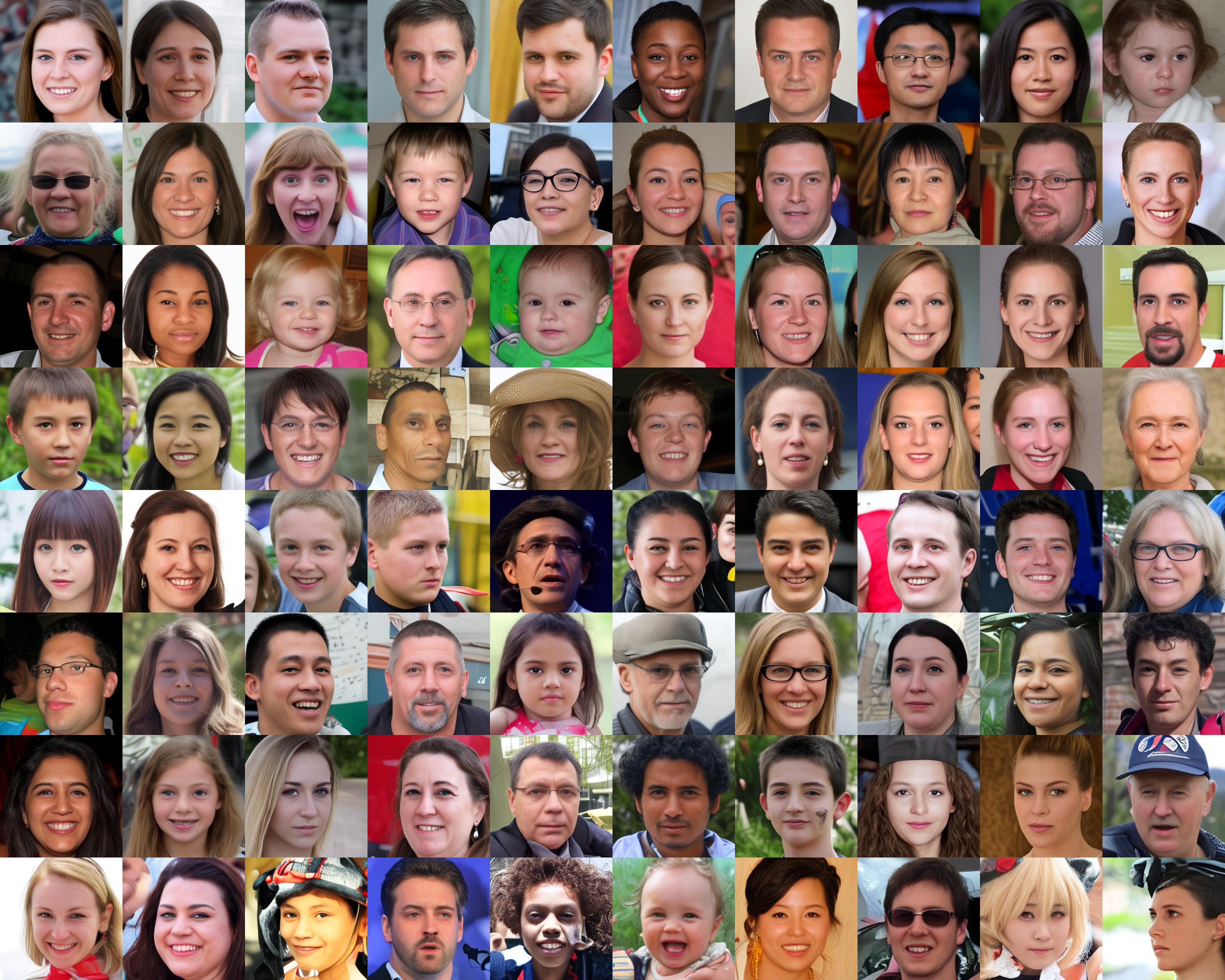}
    \caption{Unconditional $256 \times 256$ samples on FFHQ. Sampling temperatures $\kappa$ are linearly interpolated from 0.5 to 1.0 (left to right).}
    \label{fig:big_ffhq}
\end{figure}
\section{Beyond Image Synthesis: Ising Model Generation}
\label{appendix_ising}

Finally, we evaluate BFM on the Ising model, a canonical probabilistic model in statistical physics defined over interacting spin variables $\mathbf{s} \in \{+1,-1\}^{D}$ \citep{ising1925beitrag}. This experiment demonstrates that BFM is not restricted to image generation, but also applies to general binary scientific data.

\begin{figure*}[h!]
    \includegraphics[width=\textwidth]{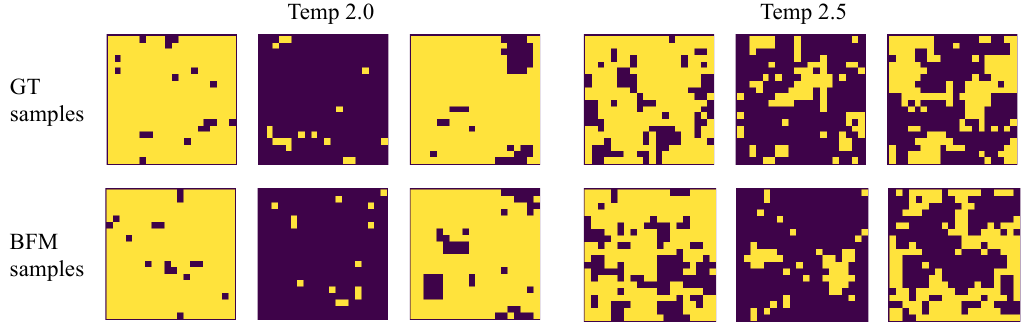}
    \caption{Comparison of $20\times20$ Ising lattice samples generated by BFM and ground-truth samples at temperatures 2.0 and 2.5.}
    \label{fig:ising_model_20x20}
\end{figure*}

To obtain reference Ising samples at different temperatures, we use the Wolff cluster algorithm \cite{wolff1989collective}. We consider a $20\times20$ 2D lattice Ising model with periodic boundary conditions, where each spin interacts with its four nearest neighbors. The Markov chain is run for 1000 steps to reach thermal equilibrium. We collect 500 training samples and 100 validation samples at each temperature from 1.50 to 3.10. We then train a conditional BFM for 100K iterations with batch size 512. During inference, we use 64 sampling steps to generate Ising samples from pure noise.

Fig.~\ref{fig:ising_model_20x20} shows representative comparisons between BFM-generated samples and Wolff-cluster reference samples at temperatures 2.0 and 2.5. BFM generates high-quality Ising configurations that are visually consistent with the reference samples. 

\begin{figure*}[htbp]
    \centering
    \begin{subfigure}[b]{0.28\textwidth}
        \includegraphics[width=\textwidth]{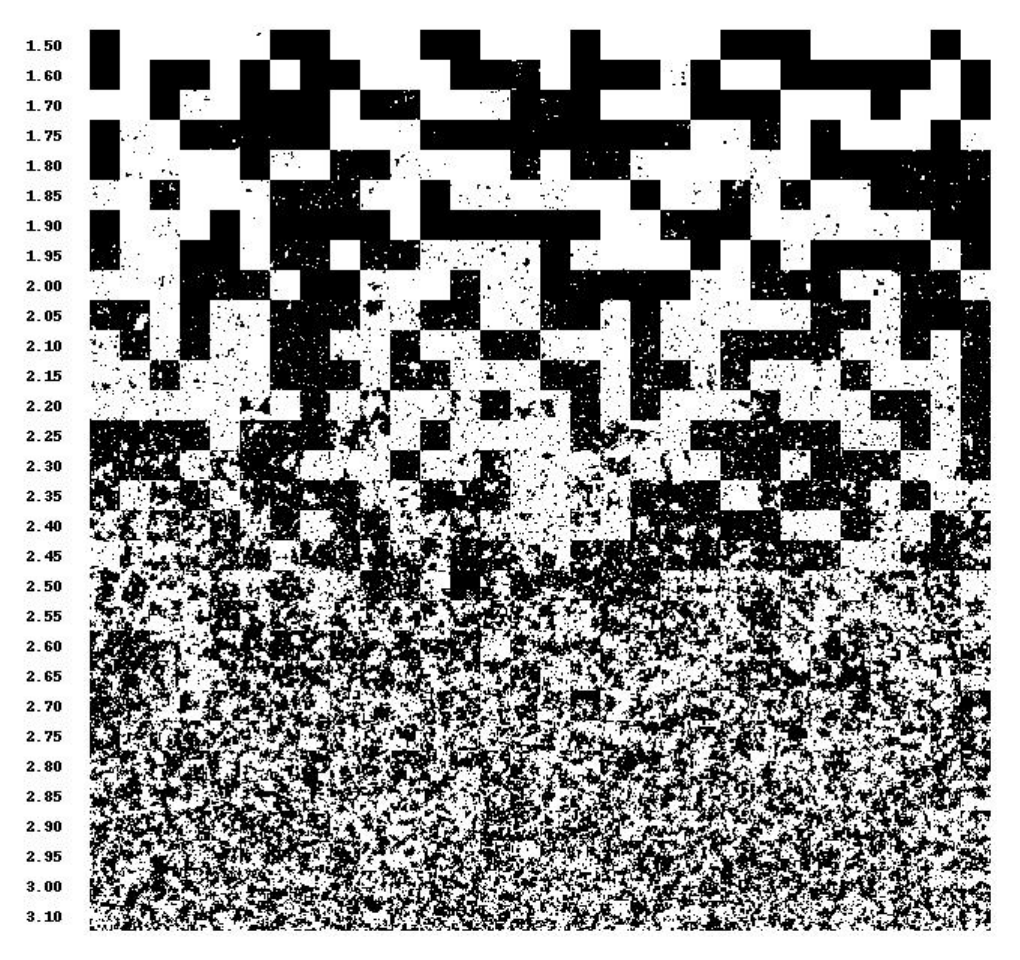}
        \caption{BFM Ising Samples}
        \label{fig:ising_sample_full}
    \end{subfigure} 
    \hfill
    \begin{subfigure}[b]{0.28\textwidth}
        \includegraphics[width=\textwidth]{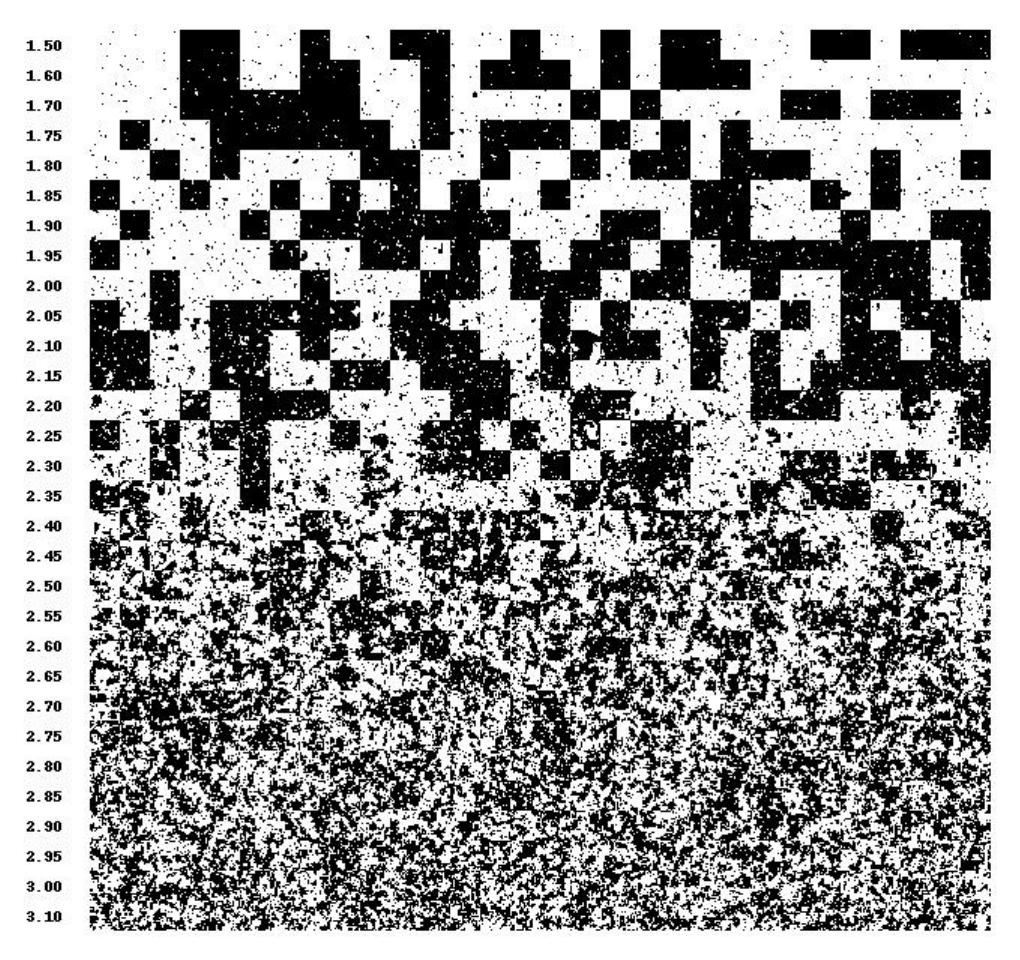}
        \caption{Wolff Cluster Ising Samples}
        \label{fig:ising_gt_full}
    \end{subfigure}
    \hfill
    \begin{subfigure}[b]{0.32\textwidth}
        \includegraphics[width=\textwidth]{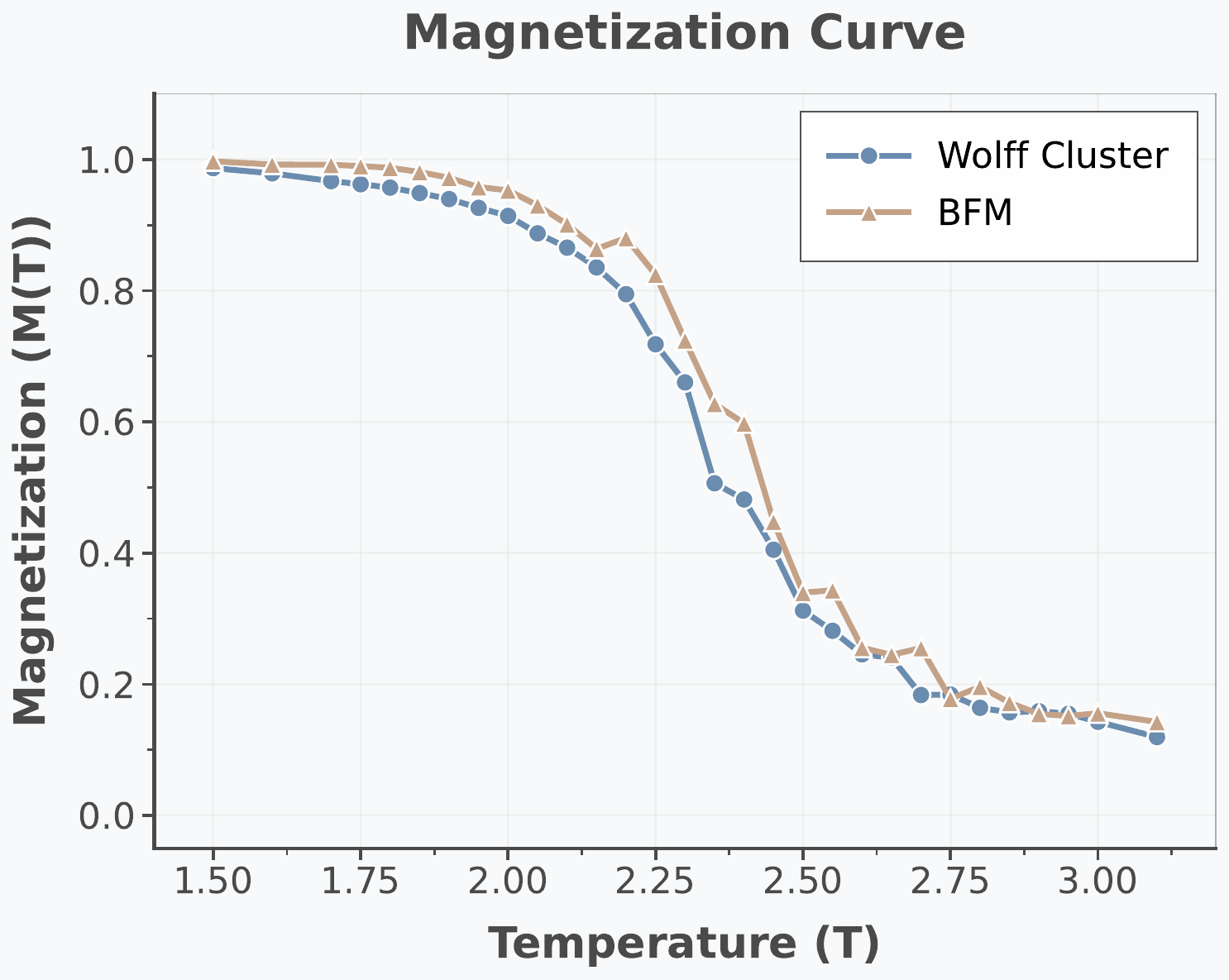}
        \caption{Magnetization Curve}
        \label{fig:ising_magnetization}
    \end{subfigure}
    \caption{Comprehensive evaluation of Ising model generation across the continuous temperature spectrum (1.50 to 3.10). (a) and (b) show the visual evolution of spin lattices generated by BFM and the ground-truth Wolff cluster~\cite{wolff1989collective} algorithm, respectively. (c) Quantitative comparison of the absolute magnetization $M(T)$, demonstrating that BFM accurately captures the thermodynamic properties.}
    \label{fig:ising_comprehensive}
\end{figure*}

To further investigate our model's ability to simulate the macroscopic thermodynamic behavior of the Ising system, we extend our evaluation across the entire temperature spectrum in Fig.~\ref{fig:ising_comprehensive}. In statistical physics, visual inspection is often insufficient; a generative model must faithfully capture the underlying Boltzmann distribution. To rigorously substantiate our claim, we calculate the absolute magnetization $M(T)$ of the generated samples as a quantitative physical observable. 

Visually, the BFM-generated samples (Fig.~\ref{fig:ising_comprehensive}a) smoothly capture the structural evolution of the spin lattices from highly ordered to completely disordered states, strictly mirroring the Wolff cluster ground truth (Fig.~\ref{fig:ising_comprehensive}b). More importantly, the quantitative magnetization curve of BFM (Fig.~\ref{fig:ising_comprehensive}c) tightly tracks the ground truth simulations. This quantitative alignment proves that BFM learns the complex physical interaction rules of the system rather than merely memorizing visual training patterns.

\section{Discussion on Cross-step Sampling Consistency of BFM}
\label{sec:cross_step_consistency}
In this section, we discuss the reason why BFM exhibits better sampling quality than BLD when using different NFE from the same trained checkpoint. This property is particularly important in practical applications, as it allows for flexible adjustment of sampling speed and quality trade-offs without the need for retraining the model.
We follow the notation used in BLD to explain this issue. In the sampling scenario of Fig.~\ref{fig:low_nfe_visual_comparison}, when we sample with $K$ NFE from a model trained with $N$ steps, we need to perform cross-step sampling with stride $m = N / K$. According to Eq.~(9) in BLD, the posterior transition can be written as
\begin{equation}
q ( \mathbf { z } ^ { t - m } | \mathbf { z } ^ { t } , \mathbf { z } ^ { 0 } ) = \frac { q ( \mathbf { z } ^ { t } | \mathbf { z } ^ { t - m } , \mathbf { z } ^ { 0 } ) q ( \mathbf { z } ^ { t - m } | \mathbf { z } ^ { 0 } ) } { q ( \mathbf { z } ^ { t } | \mathbf { z } ^ { 0 } ) } .
\end{equation}
In the official BLD code implementation, the evidence transition term is implemented as
\begin{equation}
 q ( \mathbf { z } ^ { t } | \mathbf { z } ^ { t - m } , \mathbf { z } ^ { 0 } ) = \mathcal{B}(\mathbf { z } ^ { t };\mathbf { z }^{t-m}(1-\beta^{t-m})+0.5\beta^{t-m}),
\end{equation}
which effectively uses only the last-step noise parameter $\beta^{t-m}$ for this cross-step update; however, the corresponding exact evidence transition should be
\begin{equation}
q(\mathbf { z }^t \mid \mathbf { z }^{t-m}) = \mathcal{B}\left(\mathbf { z }^t; \left(\prod_{j=t-m+1}^t (1-\beta^{j})\right) \mathbf { z }^{t-m} + 0.5\left(1 - \prod_{j=t-m+1}^t (1-\beta^{j})\right)\right)
\label{eq:bld_corrected_cross_step}
\end{equation}

In contrast, in the framework of BFM, the cross-step posterior sampling is naturally supported by our derived posterior transition in \eqref{eq:one_step_transition} and \eqref{eq:posterior}, which can be directly applied to any arbitrary cross-step posterior transition with slight modification. The only modification is to replace the original $N$-step decreasing schedule of $t_{\tau}\in[0,1]$ from $1$ to $0$ with a new schedule containing $K$ NFEs at sampling time. We refer to this property as the \textbf{self-consistency of BFM under cross-step sampling}.

Here we explain why BFM has this self-consistency property. The evidence transition $q(X_{t_\tau} | X_{t_{\tau-1}}, X_{t_0})$ from $X_{t_{\tau-1}}$ to $X_{t_{\tau}}$ depends on $\gamma_{t_\tau}$, we use the marginal distribution from $X_{t_0}$ to $X_{t_\tau}$ and $X_{t_{\tau-1}}$ to derive $\gamma_{t_\tau}$ and find it only depends on $t_{\tau}$ and $t_{\tau-1}$ as \eqref{eq:t_gamma_t_tau} shows. Therefore, when we perform cross-step sampling with different NFE, we can simply re-calculate $\gamma_{t_\tau}$ based on the new $t_{\tau}$ and $t_{\tau-1}$ without any approximation, which ensures the correctness of the posterior transition. This is the key mechanism of BFM's self-consistency property in cross-step sampling.
The self-consistency property makes BFM more flexible and robust in cross-step sampling scenarios.

\paragraph{Evaluation with Corrected BLD.}
We further evaluate BLD~\cite{wang2023binary}
using the corrected cross-step transition
in Eq.~\eqref{eq:bld_corrected_cross_step}.
Both BFM and corrected BLD are evaluated on
LSUN Churches $256\times256$ using the same
pretrained binary autoencoder, evaluation protocol,
and 50,000 generated samples at each NFE.

\begin{table}[htbp]
\centering
\caption{Low-NFE comparison between BFM and corrected
BLD on LSUN Churches $256\times256$.
Relative reduction denotes the FID improvement
of BFM over corrected BLD.}
\label{tab:corrected_bld}
\begin{tabular}{cccc}
\toprule
NFE & BFM FID $\downarrow$ &
Corrected BLD FID $\downarrow$ &
Reduction $\uparrow$ \\
\midrule
16  & \textbf{10.41} & 10.49 & 0.76\% \\
32  & \textbf{8.28}  & 8.37  & 1.08\% \\
64  & \textbf{7.78}  & 7.98  & 2.51\% \\
128 & \textbf{7.70}  & 7.85  & 1.91\% \\
256 & \textbf{7.59}  & 7.79  & 2.57\% \\
\bottomrule
\end{tabular}
\end{table}

Correcting the BLD cross-step transition substantially
improves its low-NFE performance compared with the
original implementation. Nevertheless, BFM consistently
achieves lower FID across all evaluated NFEs, with
relative reductions ranging from 0.76\% to 2.57\%.

These results indicate that the severe low-NFE
degradation observed in the original BLD implementation
is largely attributable to its cross-step transition
approximation, rather than an inherent limitation of
discrete Markov diffusion. BFM remains competitive
with the corrected baseline while retaining a direct
analytical construction for arbitrary time intervals.

\section{Additional Implementation Details}
\label{appendix_implementation_details}

In this section, we provide detailed network configurations and training hyperparameters for all tasks, as summarized in Table~\ref{tab:hyperparameters}. For high-dimensional image synthesis (LSUN and FFHQ), the training process operates within a discrete latent space. Regarding these latent representations, we utilize the official pre-trained Binary Autoencoder (BAE) provided by BLD~\cite{wang2023binary}. Rather than training a dedicated autoencoder for each specific dataset, we adopt their general-purpose \textbf{C64} model, which was trained on over 600 million images from the LAION dataset. This robust autoencoder explicitly compresses a high-resolution $256 \times 256$ image into a compact $16 \times 16 \times 64$ binary tensor, providing a highly generalized discrete space for our Bernoulli Flow Models. The pre-trained checkpoints for this autoencoder (including both the encoder and decoder) are officially provided by the BLD authors and can be accessed publicly\footnote{\url{https://drive.google.com/drive/folders/1CFlixXLnEHZ0jaRLXLS4d_bJMZwG-Iih}}. To ensure a fair comparison and robust performance in modeling these binary latents, our core generative network directly adopts the Transformer-based backbone proposed in BLD. While this standard configuration is seamlessly applied to the latent image representations, we introduce minor structural adaptations, specifically to the sequence length and positional embedding layers, to properly accommodate the unique spatial formats of Binarized MNIST and the 2D Ising model. Unlike the high-dimensional image tasks, Binarized MNIST~\cite{deng2012mnist} and the Ising model~\cite{ising1925beitrag} do not rely on any encoder-decoder mappings; our models are trained directly on their native binary spaces. For the Ising model specifically, the spin configurations were synthetically simulated using the Wolff cluster~\cite{wolff1989collective} algorithm, yielding a custom dataset of 15000 training samples and 3000 testing samples. Across all datasets, we maintain a consistent optimization strategy using the Adam optimizer coupled with a linear learning rate warmup.

\begin{table*}[htbp]
    \centering
    \caption{Configuration and training details for BFM across different datasets.}
    \label{tab:hyperparameters}
    \begin{tabular}{lccc}
        \toprule
        \textbf{Component} & \textbf{LSUN \& FFHQ} & \textbf{Binarized MNIST} & \textbf{Ising Model} \\
        \midrule
        Transformer layers      & 24                 & 24                 & 24                 \\
        Attention heads         & 12                 & 12                 & 12                 \\
        Embedding dimension     & 768                & 768                & 768                \\
        Sequence length (block size) & 256           & 28                 & 20                 \\
        Training iterations     & 500K               & 500K               & 100K               \\
        Batch size              & 96                 & 512                & 512                \\
        Optimizer               & Adam               & Adam               & Adam               \\
        Learning rate           & $2 \times 10^{-4}$ & $2 \times 10^{-4}$ & $2 \times 10^{-4}$ \\
        Weight decay            & 0.0                & 0.0                & 0.0                \\
        Warmup iterations       & 10,000             & 10,000             & 10,000             \\
        Hardware                & 2$\times$ RTX 4090 & 2$\times$ RTX 4090 & 2$\times$ RTX 4090 \\
        \bottomrule
    \end{tabular}
\end{table*}

\subsection{Compute Resources and Inference Efficiency}
\label{appendix_compute_resources}

To improve reproducibility and address the computational resource requirements of our experiments, we further report the inference-time compute cost for each task. 
All measurements are conducted with batch size $1$ and $64$ sampling steps, corresponding to the single-sample generation latency rather than throughput under large-batch sampling. 
For each benchmark, we first perform $5$ warm-up runs to exclude CUDA initialization and caching overhead, and then average the wall-clock time over $20$ repeated runs. 
We enforce CUDA synchronization before and after each measured run. For LSUN and FFHQ, we report a shared Image-256 setting because both tasks use the same $256 \times 256$ image resolution and the same $16 \times 16 \times 64$ binary latent representation produced by the pre-trained Binary Autoencoder (BAE). 
This benchmark includes both binary latent sampling and decoding through the BAE decoder. 
In contrast, Binarized MNIST and the Ising model are sampled directly in their native binary spaces without any encoder-decoder module.

\begin{table}[h]
\centering
\caption{
\textbf{Single-sample inference resource usage.}
We report the average wall-clock latency and peak GPU memory for one generated sample using $64$ sampling steps. 
The Image-256 setting is shared by LSUN and FFHQ since they use the same image resolution and binary latent format.
}
\label{tab:single_sample_compute}
\resizebox{\linewidth}{!}{
\begin{tabular}{lccccc}
\toprule
Task & Representation & Batch Size & Sampling Steps & Avg. Time / Sample (s) & Peak GPU Memory (GB) \\
\midrule
LSUN / FFHQ Image-256 & Binary latent + BAE decoder & 1 & 64 & 0.3837 & 1.0743 \\
Binarized MNIST & Native binary space & 1 & 64 & 0.3507 & 0.6482 \\
Ising Model & Native binary space & 1 & 64 & 0.3531 & 0.6480 \\
\bottomrule
\end{tabular}
}
\end{table}

As shown in Table~\ref{tab:single_sample_compute}, all tasks can be sampled with less than $1.1$ GB of peak GPU memory under the single-sample setting. 
The Image-256 benchmark has slightly higher latency and memory consumption because it operates on a higher-dimensional binary latent space and additionally invokes the BAE decoder to reconstruct RGB images. 
Binarized MNIST and Ising require less memory since they are generated directly in low-dimensional binary spaces. 
These measurements indicate that our $64$-step sampler remains lightweight at inference time across both image-latent generation and native binary-domain generation.

\section{Bootstrap Uncertainty Analysis of FID}
\label{app:bootstrap_fid}

To assess the statistical uncertainty of the FID
comparisons, we perform bootstrap analysis on
LSUN Bedrooms, LSUN Churches, and FFHQ.
For each method, we generate 50,000 samples and
perform 50 bootstrap resamples of the generated
Inception features with replacement. Within each
dataset, BFM and BLD use identical real-data
statistics, feature extractors, sample counts,
and bootstrap protocols.

\begin{table}[htbp]
\centering
\small
\caption{Bootstrap uncertainty analysis of FID on
LSUN Bedrooms, LSUN Churches, and FFHQ.
We report the full-sample FID, bootstrap mean
$\pm$ standard deviation, and 95\% bootstrap
percentile interval from 50 resamples.}
\label{tab:bootstrap_fid}
\begin{tabular}{llccc}
\toprule
Dataset & Method &
Full FID $\downarrow$ &
Boot. Mean $\pm$ Std. $\downarrow$ &
95\% Interval \\
\midrule
\multirow{2}{*}{Bedrooms}
& BFM & \textbf{5.86} & $\mathbf{6.14 \pm 0.04}$ & [6.03, 6.22] \\
& BLD & 6.21 & $6.48 \pm 0.04$ & [6.39, 6.55] \\
\midrule
\multirow{2}{*}{Churches}
& BFM & \textbf{5.32} & $\mathbf{5.59 \pm 0.04}$ & [5.50, 5.67] \\
& BLD & 5.47 & $5.71 \pm 0.03$ & [5.66, 5.77] \\
\midrule
\multirow{2}{*}{FFHQ}
& BFM & \textbf{10.87} & $\mathbf{11.21 \pm 0.08}$ & [11.11, 11.36] \\
& BLD & 11.46 & $11.79 \pm 0.07$ & [11.71, 11.94] \\
\bottomrule
\end{tabular}
\end{table}

As shown in Table~\ref{tab:bootstrap_fid},
BFM consistently achieves lower bootstrap mean
FID than BLD across all three datasets.
All bootstrap standard deviations are below
0.08, indicating limited finite-sample variability.
The 95\% intervals are separated on Bedrooms
and FFHQ, while those on Churches overlap slightly.

The bootstrap means are higher than the
corresponding full-sample FIDs because resampling
with replacement reduces the number of unique
generated samples. This effect applies to both
methods under the same evaluation protocol.

Overall, the bootstrap results preserve the
relative performance ranking between BFM and BLD
under finite-sample resampling.

\end{document}